\documentclass{article} 
\usepackage{amssymb}
\usepackage{amsmath}
\usepackage{amsthm}
\usepackage[a4paper]{geometry}
\usepackage[round]{natbib} 
 \usepackage[colorlinks=true, linkcolor=blue]{hyperref}

\usepackage{amssymb}
\usepackage{amsmath,amscd}
\usepackage{amsthm}
\usepackage{cancel}     \usepackage{centernot}  \usepackage{graphicx}   \usepackage{soul}       \usepackage{color}      \usepackage{enumitem}   \usepackage{commath}    \usepackage{algorithm}  \usepackage{subcaption} \usepackage{booktabs}

\newtheorem{theorem}{Theorem}[section]
\newtheorem{thm}[theorem]{Theorem}
\newtheorem{lemma}[theorem]{Lemma}
\newtheorem{corollary}[theorem]{Corollary}
\newtheorem{cor}[theorem]{Corollary}
\newtheorem{prop}[theorem]{Proposition}

\newtheorem*{theorem*}{Theorem}
\newtheorem*{lemma*}{Lemma}
\newtheorem*{cor*}{Corollary}
\newtheorem*{prop*}{Proposition}
\newtheorem*{conjecture*}{Conjecture}

\theoremstyle{definition}
\newtheorem{defn}{Definition}[section]
\newtheorem{definition}{Definition}[section]
\newtheorem*{definition*}{Definition}
\theoremstyle{definition}

\theoremstyle{definition}
\newtheorem{ex}{Example}
\theoremstyle{remark}
\newtheorem*{ex*}{Example}
\theoremstyle{definition}

\theoremstyle{definition}
\newtheorem*{assm*}{Assumption}
\theoremstyle{remark}
\newtheorem{remark}{Remark}[section]
\theoremstyle{remark}
\newtheorem*{remark*}{Remark}

\DeclareFontFamily{U}{mathx}{\hyphenchar\font45}
\DeclareFontShape{U}{mathx}{m}{n}{
      <5> <6> <7> <8> <9> <10> gen * mathx
      <10.95> mathx10 <12> <14.4> <17.28> <20.74> <24.88> mathx12
      }{}
\DeclareSymbolFont{mathx}{U}{mathx}{m}{n}
\DeclareMathSymbol{\intop}  {1}{mathx}{"B3}

\DeclareFontFamily{U}{mathx}{\hyphenchar\font45}
\DeclareFontShape{U}{mathx}{m}{n}{
      <5> <6> <7> <8> <9> <10>
      <10.95> <12> <14.4> <17.28> <20.74> <24.88>
      mathx10
      }{}
\DeclareSymbolFont{mathx}{U}{mathx}{m}{n}
\DeclareFontSubstitution{U}{mathx}{m}{n}
\DeclareMathAccent{\widecheck}{0}{mathx}{"71}
\DeclareMathAccent{\wideparen}{0}{mathx}{"75}

\newcommand\indep{\independent}
\newcommand\independent{\protect\mathpalette{\protect\independenT}{\perp}}
\def\independenT#1#2{\mathrel{\rlap{$#1#2$}\mkern4mu{#1#2}}}

\let\temp\phi
\let\phi\varphi
\let\varphi\temp

\newcommand{\pr}{\mathbb{P}}

\newcommand{\R}{\mathbb{R}}

\newcommand{\normalN}{\mathcal{N}}
            \newcommand{\given}{\,|\,}

\renewcommand{\norm}[1]{\Vert#1\Vert}

\newcommand{\eps}{\varepsilon}

\DeclareMathOperator{\supp}{supp}

\DeclareMathOperator{\dist}{dist}

\usepackage{multirow}
\usepackage{authblk}

\newcommand{\param}{\theta}
\newcommand{\params}{\Theta}
\newcommand{\probs}{\mathcal{P}}
\newcommand{\parametrization}{\pi}

\newcommand{\funcs}{\mathcal{F}}

\DeclareMathOperator{\Ext}{Ext}
\DeclareMathOperator{\RELU}{ReLU}
\DeclareMathOperator{\LRELU}{LReLU}
\DeclareMathOperator{\Aff}{Aff}

\DeclareMathOperator{\nbhd}{ne}

\newcommand{\prob}{\pr}

\newcommand{\ncomp}{J}

\title{Identifiability of deep generative models\\without auxiliary information}
\author[]{Bohdan Kivva$^\dag$\footnote{Equal contribution}\;}
\author[]{Goutham Rajendran$^{\dag*}$}
\author[]{Pradeep Ravikumar$^\ddag$}
\author[]{Bryon Aragam$^\dag$}
\affil[]{$^\dag$\emph{University of Chicago}, $^\ddag$\emph{Carnegie Mellon University}}

\begin{document}

\maketitle
\setcounter{footnote}{0}

\begin{abstract}
    We prove identifiability of a broad class of deep latent variable models that (a) have universal approximation capabilities and (b) are the decoders of variational autoencoders that are commonly used in practice. Unlike existing work, our analysis does not require weak supervision, auxiliary information, or conditioning in the latent space. Specifically, we show that for a broad class of generative (i.e. unsupervised) models with universal approximation capabilities, the side information $u$ is not necessary: We prove identifiability of the entire generative model where we do not observe $u$ and only observe the data $x$. The models we consider match autoencoder architectures used in practice that leverage mixture priors in the latent space and ReLU/leaky-ReLU activations in the encoder, such as VaDE and MFC-VAE. Our main result is an identifiability hierarchy that significantly generalizes previous work and exposes how different assumptions lead to different ``strengths'' of identifiability, and includes certain ``vanilla'' VAEs with isotropic Gaussian priors as a special case. For example, our weakest result establishes (unsupervised) identifiability up to an affine transformation, and thus partially resolves an open problem regarding model identifiability raised in prior work. These theoretical results are augmented with experiments on both simulated and real data.

\end{abstract}

\section{Introduction}
\label{sec:intro}

One of the key paradigm shifts in machine learning (ML) over the past decade has been the transition from handcrafted features to automated, data-driven representation learning, typically via deep neural networks. One complication of automating this step in the ML pipeline is that it is difficult to provide guarantees on what features will (or won't) be learned.
As these methods are being used in high stakes settings such as medicine, health care, law, and finance where accountability and transparency are not just desirable but often legally required, it has become necessary to place representation learning on a rigourous scientific footing. In order to do this, it is crucial to be able to discuss ideal, target features and the underlying representations that define these features. As a result, the ML literature has begun to move beyond consideration solely of downstream tasks (e.g. classification, prediction, sampling, etc.) in order to better understand the structural foundations of deep models.

Deep generative models (DGMs) such as variational autoencoders (VAEs) \citep{kingma2013auto,rezende2014stochastic} are a prominent example of such a model, and
are a powerful tool for unsupervised learning of latent representations, useful for a variety of downstream tasks such as sampling, prediction, classification, and clustering.
Despite these successes,
training DGMs is an intricate task: They are susceptible to posterior collapse and poor local minima \citep{yacoby2020failure,dai2020usual,he2018lagging,wang2021posterior},
and characterizing their latent space remains a difficult problem \citep[e.g.][]{klys2018learning,van2017neural}.
For example, does the latent space represent semantically meaningful or practically useful features? Are the learned representations stable, or are they simply artifacts of peculiar choices of hyperparameters?
These questions have been the subject of numerous studies in recent years \citep[e.g.][]{schott2021visual,luise2020generalization,locatello2019challenging,bansal2021revisiting,csiszarik2021similarity,lenc2015understanding}, and 
in order to better understand the behaviour of these models and address these questions, the machine learning literature has recently turned its attention to fundamental identifiability questions \citep{khemakhem2020variational,damour2020underspecification,wang2021posterior}. Identifiability is a crucial primitive in machine learning tasks that is useful for probing stability, consistency, and robustness.
Without identifiability, the output of a model can be unstable and unreliable, in the sense that retraining under small perturbations of the data and/or hyperparameters may result in wildly different models.\footnote{Formally, identifiability means the parametrization of the model is injective. See Section~\ref{sec:prelim} for details.}
In the context of deep generative models, the model output of interest is the latent space and the associated representations induced by the model.

In this paper, we revisit the identifiability problem in deep latent variable models and prove a surprising new result: Identifiability is possible under commonly adopted assumptions \emph{and} without conditioning in the latent space, or equivalently, without weak supervision or side information in the form of auxiliary variables. This contrasts a recent line of work that has established fundamental new results regarding the identifiability of VAEs that requires conditioning on an auxiliary variable $u$ that renders each latent dimension conditionally independent \citep{khemakhem2020variational}. While this result has been generalized and relaxed in several directions \citep{halva2020hidden,halva2021disentangling,khemakhem2020ice,li2019identifying,pmlr-v139-mita21a,sorrenson2019disentanglement,yang2021nonlinear,klindt2020towards,brehmer2022weakly}, fundamentally these results still crucially rely on the side information $u$. We show that this is in fact unnecessary---confirming existing empirical studies \citep[e.g][]{willetts2021don,falck2021multi}---and do so without sacrificing any representational capacity. What’s more, the model we analyze is closely related to deep architectures that have been widely used in practice \citep{dilokthanakul2016deep,falck2021multi,jiang2016variational,johnson2016composing,lee2020meta,li2018learning,willetts2019disentangling,lee2020meta}: We show that there is good reason for this, and provide new insight into the properties of these models and support for their continued use.

\paragraph{Overview}
More specifically, we consider the following generative model for observations $x$:
\begin{align}
\label{eq:defn:nica}
x = f(z) + \eps, \quad
x=(x_{1},\ldots,x_{n})\in\R^{n}, \quad
z=(z_{1},\ldots,z_{m})\in\R^{m},
\end{align}
where the latent variable $z$ follows a Gaussian mixture model (GMM),\footnote{See Remark~\ref{rem:gmm} for extensions to more general mixture priors.}
$f:\R^{m}\to\R^{n}$ is a piecewise affine nonlinearity such as a ReLU network, and $\eps\in\R^{n}$ is independent, random noise.\footnote{Our results include the noiseless case $\eps=0$ as a special case.} We do not assume that the number of mixture components, nor the architecture of the ReLU network, are known in advance, nor do we assume that $z$ has independent components. Both the mixture model and neural network may be arbitrarily complex, and we allow for the discrete hidden state that generates the latent mixture prior to be high-dimensional and dependent. This includes both vanilla VAEs (i.e. with a standard isotropic Gaussian prior) and classical ICA models (i.e. for which the latent variables are mutually independent) as special cases. Since both $z$ and $f$ are allowed to be arbitrarily complex, the model \eqref{eq:defn:nica} has universal approximation capabilities, which is crucial for modern applications.

This model has been widely studied in the literature from a variety of different perspectives:
\begin{itemize}
\item \emph{Nonlinear ICA.} When the $z_{i}$ are mutually independent, \eqref{eq:defn:nica} recovers the standard nonlinear ICA model that has been extensively studied in the literature \citep{hyvarinen1999nonlinear,achard2005identifiability,zhang2008minimal,hyvarinen2017nonlinear,hyvarinen2019nonlinear,hyvarinen2016unsupervised}. 
Although our most general results do not make independence assumptions, our results cover nonlinear ICA as a special case (see Section~\ref{sec:main:cases} for more discussion).
\item \emph{VAE with mixture priors.} When the prior over $z$ is a mixture model (e.g. such as a GMM), the model \eqref{eq:defn:nica} is closely related to popular autoencoder architectures such as
VaDE \citep{jiang2016variational},
SVAE \citep{johnson2016composing},
GMVAE \citep{dilokthanakul2016deep},
DLGMM \citep{nalisnick2016approximate},
VampPrior \citep{tomczak2018vae},
MFC-VAE \citep{falck2021multi},
etc.
Although such VAEs with mixture priors have been used extensively in applications, theoretical results are missing.
\item \emph{Warped mixtures.} Another closely related model is the warped mixture model of \citet{iwata2013warped}, which is a Bayesian version of \eqref{eq:defn:nica}. 
Once again, theoretical guarantees for these models are lacking.
\item \emph{iVAE.} Finally, \eqref{eq:defn:nica} is also the basis of
the iVAE model introduced by \citet{khemakhem2020variational}, where identifiability (up to certain equivalences) is proved when there is an additional auxiliary variable $u$ that is observed such that $z_{i}\indep z_{j}\given u$.
\end{itemize}

\begin{table}
\centering
\begin{tabular}{cccc}
 \toprule
 Assumptions on $f$ & Assumptions on $Z$ & Theoretical guarantees & Result\\
 \midrule
 \ref{assm:mixture} & \ref{assm:relu}, \ref{assm:invib} & $\mathbb{P}(Z)$ identifiable up to & Theorems \\
 && an {affine transformation} & \ref{thm:main:1d}(a), \ref{thm:main:highd}(a)\\
 [1ex]
\ref{assm:mixture} & \ref{assm:relu}, \ref{assm:inv} & $\mathbb{P}(Z)$ and $f$ up to  identifiable&Theorems \\
 &&an {affine transformation} & \ref{thm:main:1d}(c), \ref{thm:main:highd}(d)\\
 [1ex]
\ref{assm:mixture}, \ref{assm:latent} & \ref{assm:relu}, \ref{assm:inv} & $\mathbb{P}(Z)$ and $f$ identifiable up to  & Theorems \\
 &&{permutation, scaling and translation} & \ref{thm:main:1d}(b), \ref{thm:main:highd}(b)\\
 [1ex]
\ref{assm:mixture}, \ref{assm:latent}, \ref{assm:latentdag} & \ref{assm:relu}, \ref{assm:inv}
 & $\mathbb{P}(U, Z)$ and $f$ are identifiable up to & Theorems \\
 &&{permutation, scaling and translation} & \ref{thm:main:highd}(c), \ref{thm:main:highd}(d)\\
 [1ex]
 \bottomrule\\
\end{tabular}
\caption{Summary of results in this paper. The strength of the assumptions increases in each successive row, as do the strength of the guarantees. See Section~\ref{sec:main:results} for formal statements.}\label{tab:summary}
\end{table}

\paragraph{Contributions}
Driven by this recent interest from both applied and theoretical perspectives, our main results (Theorems~\ref{thm:main:1d},~\ref{thm:main:highd}) show that the model \eqref{eq:defn:nica} is identifiable up to various \emph{linear} equivalences, without conditioning or auxiliary information in the latent space. 
In fact, we develop a hierarchy of results under progressively stronger assumptions on the model, beginning with affine equivalence and ending up with a much stronger equivalence up to permutations only. See Table~\ref{tab:summary} for a summary.

In order to develop this hierarchy, we prove several technical results of independent interest: 
\begin{enumerate}
    \item First, we establish a novel identifiability result for nonparametric mixtures (Theorem~\ref{thm:main:npmixF2});
    \item Second, we show how to use the mixture prior to strengthen existing identifiability results for nonlinear ICA (Theorem~\ref{thm:dontneedu-strong});
    \item Third, we extend existing results \citep{kivva2021learning} on the recovery of structured multivariate discrete latent variable models to recovery under an unknown affine transformation (Theorem~\ref{thm:ivae-latentdag}).
\end{enumerate}
Our proof techniques---based on elementary tools from analytic function theory and mixture identifiability---are new and depart from existing work in this area. 
As a consequence, the analysis itself provides new insight into the structure and behaviour of deep generative models.

\paragraph{Related work}
This problem is widely studied, and has garnered significant recent interest, so we focus only on the most closely related work here.

Classical results on nonlinear ICA \citep{hyvarinen1999nonlinear} establish the nonidentifiability of the general model (i.e. without restrictions on $z$ and $f$); see also \citet{darmois1951analyse,jutten2003advances}.
More recently, \citet{khemakhem2020variational} proved a major breakthrough by showing that given side information $u$, identifiability of the entire generative model is possible up to certain (nonlinear) equivalences. Since this pathbreaking work, many generalizations have been proposed \citep{halva2020hidden,halva2021disentangling,khemakhem2020ice,li2019identifying,pmlr-v139-mita21a,sorrenson2019disentanglement,yang2021nonlinear,klindt2020towards,brehmer2022weakly}, all of which require some form of auxiliary information. 
Other approaches to identifiability include various forms of weak supervision such as contrastive learning \citep{zimmermann2021contrastive}, group-based disentanglement \citep{locatello2020weakly}, and independent mechanisms \citep{gresele2021independent}.
Non-identifiability has also been singled out as a contributing factor to practical issues such as posterior collapse in VAEs \citep{wang2021posterior,yacoby2020failure}.

Our approach is to avoid additional forms of supervision altogether, and enforce identifiability in a purely unsupervised fashion. Recent work along these lines includes \citet{wang2021posterior}, who propose to use Brenier maps and input convex neural networks, and \citet{moran2021identifiable} who leverage sparsity and an anchor feature assumption. Aside from different assumptions, the main difference between this line of work and our work is that their work only identifies the latent space $P(Z)$, whereas our focus is on jointly identifying \emph{both} $P(Z)$ and $f$. In fact, we provide a decoupled set of assumptions that allow $f$ or $P(Z)$ or both to be identified. Thus, we  partially resolve in the affirmative an open problem regarding model identifiability raised by the authors in their discussion.

Another distinction between this line of work and the current work is our focus on architectures and modeling assumptions that are \emph{standard} in the deep generative modeling literature, specifically ReLU nonlinearities and mixture priors.
As noted above, there is a recent tradition of training variational autoencoders with mixture priors \citep{dilokthanakul2016deep,falck2021multi,jiang2016variational,johnson2016composing,lee2020meta,li2018learning,willetts2019disentangling,lee2020meta}. Our work builds upon this empirical literature, showing that there is good reason to study such models: Not only have they been shown to be more effective compared to vanilla VAEs, we show that they have appealing theoretical properties as well. In fact, recent work \citep{willetts2021don,falck2021multi} has observed precisely the identifiability phenomena studied in our paper, however, this work lacks rigourous theoretical results to explain these observations.

Another related line of work studies identification in graphical models with latent variables, albeit without any explicit connection to deep generative models \citep{pearl1992statistical,evans2016graphs,markham2020measurement,kivva2021learning}.

Finally, since a key step in our proof involves the analysis of a nonparametric mixture model (see Appendix~\ref{sec:proofdontneedu} for details), it is worth reviewing previous work in mixture models. See \citet{allman2009} for an overview.
Of particular use for the present work are \citet{teicher1963identifiability} and \citet{barndorff1965}, wherein the identifiability of Gaussian and exponential family mixtures, respectively, are proved.
Specifically for nonparametric mixtures, existing results consider product mixtures \citep{teicher1967,hall2003}, grouped observations \citep{ritchie2020consistent,vandermeulen2019operator}, symmetric measures \citep{hunter2007,bordes2006}, and separation conditions \citep{aragam2018npmix}.
For context, we note here that a discrete VAE can be interpreted as a mixture model in disguise: This is a perspective that we leverage in our proofs. We are not aware of previous work in the deep generative modeling literature that exploits this connection to prove identifiability results.

\section{Preliminaries}
\label{sec:prelim}

We first introduce the main generative model that we study and its properties, and then proceed with a brief review of identifiability in deep generative models.

\paragraph{Generative model}
The observations $x\in\R^{n}$ are realizations of a random vector $X$, and are generated according to the generative model \eqref{eq:defn:nica},
where $z\in\R^{m}$ represents realizations of an unobserved random vector $Z$. We make the following assumptions on $Z$ and $f$:\footnote{In the sequel, we will use (P\#) to index assumptions on the prior $P(Z)$, and (F\#) to index assumptions on the decoder $f$.}
\begin{enumerate}[label=(P\arabic*)]
\item\label{assm:mixture} $P(Z)$ is a (possibly degenerate) Gaussian mixture model with an unknown number of components $\ncomp\ge 1$, i.e.
\begin{align}
\label{eq:assm:mixture}
p(z) = \sum_{j=1}^{\ncomp} \lambda_{j}\phi(z;\mu_{j},\Sigma_{j}),
\quad\sum_{j=1}^{\ncomp}\lambda_{j}=1,
\quad
\lambda_j>0,
\end{align}
where $p(z)$ is the density of $P(Z)$ with respect to some base measure, and $\phi(z;\mu_{j},\Sigma_{j})$ is the gaussian density with mean $\mu_{j}$ and covariance $\Sigma_{j}$.
\end{enumerate}
\begin{enumerate}[label=(F\arabic*)]
\item\label{assm:relu} $f$ is a piecewise affine function, such as a multilayer perceptron with ReLU (or leaky ReLU) activations. 
\end{enumerate}
Recall that an affine function is a function $x\mapsto Ax+b$ for some matrix $A$.
As already discussed, special cases of this model have been extensively studied in both applications and theory, and both \ref{assm:mixture}-\ref{assm:relu} are quite standard in the literature on deep generative models and represent a useful model that is widely used in practice \citep[e.g.][]{dilokthanakul2016deep,falck2021multi,jiang2016variational,johnson2016composing,lee2020meta,li2018learning,willetts2019disentangling,lee2020meta}. In particular, when $J=1$ this is simply a classical VAE with an isotropic Gaussian prior (see Section~\ref{sec:main:cases} for more discussion).

\begin{remark}
\label{rem:gmm}
The assumption that $P(Z)$ is a GMM can be replaced with more general exponential family mixtures \citep{barndorff1965} as long as (a) the resulting mixture prior $p(z)$ is an analytic function and (b) the exponential family is closed under affine transformations.
\end{remark}

\paragraph{Universal approximation}
Under assumptions \ref{assm:mixture}-\ref{assm:relu}, the model \eqref{eq:defn:nica} has universal approximation capabilities. In fact, any distribution can be approximated by a mixture model \eqref{eq:assm:mixture} with sufficiently many components $\ncomp$ \cite[e.g.][]{nguyen2019approximations}. Alternatively, when $\ncomp$ is bounded, by taking $f$ to be a sufficiently deep and/or wide ReLU network, any distribution can be approximated by $f(Z)$ \cite[e.g.][]{lu2020universal,teshima2020coupling}, even if $f$ is invertible \citep{ishikawa2022universal}. Thus, there is no loss in representational capacity in \ref{assm:mixture}-\ref{assm:relu}. To the best of our knowledge, our results are the first to establish identifiability of \emph{both} the latent space and decoder for deep generative models \emph{without} conditioning in the latent space or weak supervision. We note that \citet{wang2021posterior} and \citet{moran2021identifiable} also propose deep architectures that identify the latent space, but not the decoder.

\paragraph{Identifiability}
A statistical model is specified by a (possibly infinite-dimensional, as in our setting) parameter space $\params$, a family of distributions $\probs$, and a mapping $\parametrization:\params\to\probs$; i.e. $\parametrization(\param)\in\probs$ for each $\param\in\params$. In more conventional notation, we define $\probs=\{p_{\param}:\param\in\params\}$, in which case $p_{\param}=\parametrization(\param)$. A statistical model is called \emph{identifiable} if the parameter mapping $\parametrization$ is one-to-one (injective). In practical applications, the strict definition of identifiability is too strong, and relaxed notions of identifiability are sufficient. Classical examples include identifiability up to permutation, re-scaling, or orthogonal transformation.
More generally, a statistical model is \emph{identifiable up to an equivalence relation $\sim$} defined on $\params$ if $\parametrization(\param)=\parametrization(\param')\implies \param \sim \param'$. 
For more details on the different notions of identifiability in deep generative models, see \cite{khemakhem2020variational,khemakhem2020ice,pmlr-v139-roeder21a}.

More precisely, we use the following definition. 
Let $f_{\sharp}P$ denote the pushforward measure of $P$ by $f$.

\begin{definition}

Let $\probs$ be a family of probability distributions on $\mathbb{R}^m$ and $\funcs$ be a family of functions $f:\mathbb{R}^m\rightarrow \mathbb{R}^n$. 
\begin{enumerate}
\item For $(P, f)\in \probs\times \funcs$ we say that the prior $P$ is \emph{identifiable (from $f_{\sharp}P$) up to an affine transformation} if for any $(P', f')\in \probs\times \funcs$ such that $f_{\sharp}P\equiv f'_{\sharp}P'$ there exists an invertible affine map $h:\mathbb{R}^m\rightarrow \mathbb{R}^m$ such that $P' = h_{\sharp}P$ (i.e., $P'$ is the pushforward measure of $P$ by $h$).
\item For $(P, f)\in \probs\times \funcs$ we say that the pair $(P, f)$ is \emph{identifiable (from $f_{\sharp}P$) up to an affine transformation} if for any $(P', f')\in \probs\times \funcs$ such that $f_{\sharp}P\equiv f'_{\sharp}P'$ there exists an invertible affine map $h:\mathbb{R}^m\rightarrow \mathbb{R}^m$ 
such that $f' = f\circ h^{-1}$ and $P' = h_{\sharp}P$.
\end{enumerate}
If the noise $\varepsilon$ has a known distribution, then $f_{\sharp}P$ is identifiable from the convolution $(f_{\sharp}P)\ast \varepsilon$. Hence, this definition can be automatically extended to the setup with known noise.
This definition also  can be extended to transformations besides affine transformations (e.g. permutations, translations, etc.) in the obvious way.
\end{definition}

Identifiability is a crucial property for a statistical model: 
Without identifiability, different training runs may lead to very different parameters, making training unpredictable and replication difficult. The failure of identifiability, also known as \emph{underspecification} and \emph{ill-posedness}, has recently been flagged in the ML literature as a root cause of many failure modes that arise in practice \citep{damour2020underspecification,yacoby2020failure,wang2021posterior}. 
As a result, there has been a growing emphasis on identification in the deep learning literature, which motivates the current work.
Finally, in addition to these reproducibility and interpretability concerns, identifiability is a key component in many applications of latent variable models including
causal representation learning \citep{scholkopf2021toward},
independent component analysis \citep{comon1994}, and
topic modeling \citep{arora2012learning,anandkumar2013overcomplete}.
See \citet{ran2017parameter} for additional discussion and examples.

\paragraph{Auxiliary information and iVAE}
It is well-known that assuming independence of the latent factors---i.e. $Z_{i}\indep Z_{j}$---is insufficient for identifiability \citep{hyvarinen1999nonlinear}.
Recent work, starting with iVAE, shows identifiability by additionally assuming that a $k$-dimensional auxiliary variable $u$ is observed such that $p(z\given u)$ is conditionally factorial, i.e. $Z_{i}\indep Z_{j}\given U$. This extra information serves to break symmetries in the latent space and is crucial to existing proofs of identifiability.

To make the connection with this work clear, observe that assumption \ref{assm:mixture} is equivalent to assuming that there is an additional hidden state $U\in\{1,\ldots,\ncomp\}$ such that $P(Z=z\given U=j)=p_{j}(z)$ and $P(U=j)=\lambda_{j}$. More generally, $U=(U_1,\ldots,U_k)$ may be multivariate. In this way, a direct parallel between our work and previous work is evident, with several crucial caveats:
\begin{itemize}
\item We do \emph{not} assume that $U$ is observed---even partially---or known in any way;
\item We allow for the $Z_{i}$ to be arbtrarily dependent even after conditioning on $U$, and this dependence need not be known;
\item We do not even require the number of states $\ncomp$ to be known, and we do not require any bounds on $\ncomp$ (e.g. iVAE requires  $\ncomp\geq m+1$). \item In the case where $U$ is multivariate (i.e $k:=\dim(U)>1$), we do not require the number of latent dimensions $k$, the state spaces, or their dependencies to be known.
\item The original iVAE paper only proves identifiability of $f$ up to a nonlinear transformation (see Lemma~\ref{obs:ivae:equiv} in Appendix~\ref{sec:ivae:equiv} for details). By contrast, we will show identifiability of $f$ up to an affine transformation, without knowing $U$. 
\end{itemize}
In order to break the symmetry without knowing anything about $U$ or its dependencies, we develop fundamentally new insights into nonparametric identifiability of latent variable models.

\section{Main results}
\label{sec:main}

For any positive integer $d$, let $[d]=\{1,\ldots,d\}$.
By \ref{assm:mixture}, we can write the model \eqref{eq:defn:nica} as follows.
Let $U=(U_1,\ldots,U_k)\in[d_1]\times\cdots[d_k]$ where $d_i:=\dim(U_i)$ and $k:=\dim(U)$; we allow $U$ to be multivariate ($k>1$) and dependent---i.e., we do not assume that the $U_i$ are marginally independent.
It follows trivially from \ref{assm:mixture} that $P(U_1=u_1,\ldots,U_k=u_k)\in\{\lambda_1,\ldots,\lambda_{\ncomp}\}$
and $\ncomp=\prod_id_i$, where we recall that $\ncomp$ is the \emph{unknown} number of mixture components in $P(Z)$. 
Denote the marginal distribution of $U$, which depends on $\lambda_j$, by $P_{\lambda}$.
The variables $(U,Z)$ are unobserved and encode the underlying latent structure:
\begin{align}
\label{eq:gmm-ivae}
\left.
\begin{aligned}
U=u
&\sim P_{\lambda}(U=u) \\
[Z\given U=u]
&\sim N(\mu_u, \Sigma_u) \\
[X\given Z=z]
&\sim f(z) + \eps,\quad \eps\sim\normalN(0,\sigma^{2})
\end{aligned}
\right\}
\implies
U\to Z\to X.
\end{align}
Here, $P_{\lambda}$ is the distribution on $U$ described above. Our goal is to identify the latent distribution $P(U,Z)$ and/or the nonlinear decoder $f$ from the marginal distribution $P(X)$ induced by \eqref{eq:gmm-ivae}. We will additionally assume throughout that $m\le n$; see Remark~\ref{remark:overcomplete} for a discussion of the overcomplete case with $m>n$.

Our main results (Theorems~\ref{thm:main:1d}-\ref{thm:main:highd}) provide a hierarchy of progressively stronger conditions under which $P(U,Z)$, $f$, or both, can be identified in progressively stronger ways. The idea is to illustrate explicitly what conditions are sufficient to identify the latent structure up to affine equivalence (the weakest notion of identifiability we consider), equivalence up to permutation, scaling, and translation, and permutation equivalence (the strongest notion of identifiability we consider, and the strongest possible for any latent variable model).

We defer the statement of the main results to Section~\ref{sec:main:results}, after the main conditions have been described. As a preview to the main results, we first present the following corollary:
\begin{cor}\label{cor:relu}
Suppose $k=\dim(U)=1$ , $J\ge 1$, $(U,Z)$ are unobserved, and $X$ is observed. 
(a) If $f$ is an invertible ReLU network, then both $P(U,Z)$ and $f$ are identifiable up to an affine transformation. (b) If $f$ 
is only weakly injective (cf. \ref{assm:invib}), then $P(U,Z)$ is still identifiable up to an affine transformation.
\end{cor}
For comparison, Corollary~\ref{cor:relu} already strengthens existing results, since $U$ is not required to be known and we are able to identify $f$. In fact, the latter answers an open question raised by \citet{wang2021posterior}.
What's more, this is just the \emph{weakest} result implied by our main results: Under stronger assumptions on the latent structure, the affine equivalence presented above can be strengthened further.

Taken together, the results in this section have the following concrete implication for practitioners: For stably training variational autoencoders, there is now compelling justification to work with a GMM prior and deep ReLU/Leaky-ReLU networks. As we saw above, this is commonly done in practice already.

\subsection{Possible assumptions on $f$}
\label{sec:main:f}

To distinguish cases where $f$ is and is not identifiable, we require the following technical definition. Recall that for sets $A, B$, $f^{-1}(A)=\{x:f(x)\in A\}$ and $f(B) = \{f(x): x\in B\}$. 
\begin{definition}\label{def:invertible}
Let $m\leq n$ (see Remark~\ref{remark:overcomplete}) and $f:\mathbb{R}^m\rightarrow \mathbb{R}^n$. 
\begin{enumerate}[label=(F\arabic*)]
\setcounter{enumi}{1}
    \item\label{assm:invib}  We say that $f$ is \emph{weakly injective}   if (i) there exists $x_0\in \mathbb{R}^n$ and $\delta>0$ s.t. $|f^{-1}(\{x\})|=1$ for every $x\in B(x_0, \delta)\cap f(\mathbb{R}^m)$, and (ii) $\{x\in \mathbb{R}^n\, :\, |f^{-1}(\{x\})| = \infty\}\subseteq f(\mathbb{R}^m)$ has measure zero with respect to the Lebesgue measure on $f(\mathbb{R}^m)$.
\item\label{assm:invae} We say that $f$ is \emph{observably injective} if $\{x\in \mathbb{R}^n\, :\, |f^{-1}(\{x\})|>1\}\subseteq f(\mathbb{R}^m)$ has measure zero with respect to the Lebesgue measure on $f(\mathbb{R}^m)$. In other words, $f$ is injective for almost every $x$ in its image $f(\mathbb{R}^m)$ (i.e. almost every ``observable'' $x$). 
\item\label{assm:inv}  We say that $f$ is \emph{injective} if $|f^{-1}(\{x\})|=1$ for every $x\in f(\mathbb{R}^m)$.
\end{enumerate}
\end{definition}

\begin{remark}
For piecewise affine functions assumption \ref{assm:invib} is weaker than assumption \ref{assm:invae}, which in turn is weaker than \ref{assm:inv}.  Therefore, for piecewise affine functions we have the chain of implications:
\[ \text{\ref{assm:inv}} \implies \text{\ref{assm:invae}} \implies \text{\ref{assm:invib}}. \]
In the sequel, we mostly focus on \ref{assm:invib} and \ref{assm:inv} for simplicity; although we prove results for \ref{assm:invae} in Appendix~\ref{app:assm:invae}. See also Remarks~\ref{remark:relunn-generic},~\ref{rem:A4}.
\end{remark}

\begin{ex}\label{ex:relu-invertible}
In general, a deep ReLU network may be either injective or observably injective, or neither (e.g. $\RELU(-\RELU(x)) = 0$).
For example, although $x\mapsto \RELU(x)$ is not injective, it is observably injective, where $\RELU(x)=\max\{0,x\}$ is the usual rectified linear unit. To see this, note that image of $\RELU$ is the set $\mathbb{R}_{\geq} = \{y\mid y\geq 0\}$, and $\RELU$ has the unique preimage for every $y\in \mathbb{R}_{>} = \{y \mid y>0\}$. Clearly, $(\mathbb{R}_{\geq}\setminus \mathbb{R}_{>}) = \{0\}$ has measure zero inside $\mathbb{R}_{\geq}$.

At the same time, $x\mapsto 0$ and $x\mapsto |x|$ are not even weakly injective.
\end{ex}

\begin{remark}\label{remark:relunn-generic}
In Appendix~\ref{sec:relu-inv}, we show that ReLU networks or Leaky ReLU networks are generically observably injective (and hence also weakly injective) under simple assumptions on their architecture.
\end{remark}

\begin{remark}\label{remark:overcomplete}
We restrict attention to the case $m\leq n$, which is a standard assumption, as it is common to think of a latent space to be a low-dimensional representation of the observed space. In the overcomplete case, i.e. when $m>n$, we believe that identifiability is unlikely unless stronger assumptions are made, or weaker notions of identifiability are considered. To see this, consider the projection $f(x, y) = x$, which is trivially affine. Then we can arbitrarily transform the $y$-coordinate without changing $P$, i.e. 
$(f\circ g)_{\sharp}P=f_{\sharp}P$, where $g(x,y)=(x,h(y))$ for any $h$.
As an example of identifiability in the overcomplete regime under stronger assumptions, when the auxiliary variable $u$ is known, 
\cite{khemakhem2020ice} show that the feature maps $f$ and $g$ in conditional energy-based models (for which $p(x\mid u) \propto \exp(f(x)^Tg(u))$) can be identified up to an affine transformation. 
\end{remark}

\subsection{Possible assumptions on $Z$}
\label{sec:main:Z}

Our weakest result requires no additional assumptions on $Z$ beyond \ref{assm:mixture}; see Corollary~\ref{cor:relu}. Under stronger assumptions, more can be concluded. As with the previous section, the assumptions presented here are not necessary, but may be imposed in order to extract stronger results.

The first condition is a mild condition that allows us to strengthen affine identifiability:
\begin{enumerate}[label=(P\arabic*)]
\setcounter{enumi}{1}
    \item\label{assm:latent} $Z_i\indep Z_j \mid U$ for all $i\neq j$ and there exist a pair of states $U = u_1$ and $U = u_2$ such that all $(\left(\Sigma_{u_1}\right)_{tt}/\left(\Sigma_{u_2}\right)_{tt}\mid t\in [m])$ are distinct. (Note that this implies $J\geq 2$).
\end{enumerate}
The second condition is more technical, and is only necessary if $k>1$ and we wish to identify $P(U)$ in addition to $P(Z)$. 
In fact, not only will we recover $P(U)$, but also the (unknown) number of hidden variables (i.e. $k$) and their state spaces (i.e. $d_j$). Note that $P(U)$ is not needed to sample from \eqref{eq:defn:nica}, as long as we have $P(Z)$. Before introducing this condition, we need a preliminary definition.
\begin{definition}\label{def:nbhd}
Let $U_{-i}$ denote $\{U_j : j\neq i\}$. We define $\nbhd(U_i) = [m]\setminus \{t: Z_t\indep U_i\mid U_{-i} \}$ and $\nbhd(Z_i)=\{t : Z_i\in\nbhd(U_t)\}$. For a subset $Z'\subset Z$, $\nbhd(Z')=\cup_{Z_i\in Z'}\nbhd(Z_i)$.
\end{definition}
The neighborhood $\nbhd(U_i)$ collects the variables $Z_t$ that depend on $U_i$ directly.

\begin{enumerate}[label=(P\arabic*)]
\setcounter{enumi}{2}
\item\label{assm:latentdag} The following conditions hold:
    \begin{enumerate}
        \item For all $Z'\subset Z$ and $u_1\ne u_2$, $P(Z'\given \nbhd(Z') = u_1)\ne P(Z'\given\nbhd(Z') = u_2)$;
\item If $P(U',Z,X)=P(U,Z,X)$, then $\dim(U')\le\dim(U)$;  and
        \item For any $U_i\ne U_j$ the set $\nbhd(U_i)$ is not a subset of $\nbhd(U_j)$.
    \end{enumerate}
\end{enumerate}
Condition~\ref{assm:latentdag} 
is a ``maximality'' condition that is adapted from \citet{kivva2021learning}: We are interested in identifying the most complex latent structure with the most number of hidden variables. This is in fact necessary since we can always merge two (or more) hidden variables into a single hidden variable without changing the joint distribution. 
Moreover, if two distinct hidden variables $U_i\ne U_j$ have the same neighborhood (or one is a subset of another), then it is known that $P(U)$ cannot be identified \citep{pearl1992statistical,evans2016graphs,kivva2021learning}. Evidently, if we seek to learn $P(U)$ in addition to $P(Z)$, then this must be avoided.
Finally, as the proof will indicate, this condition is slightly stronger than what is needed (see Remark~\ref{rem:maxim-hierarchy} for details).

\begin{remark}
\label{rem:anchor}
Condition~\ref{assm:latentdag} should be contrasted with the stronger ``anchor words'' assumption that has appeared in prior work \citep{arora2012learning, arora2013practical,moran2021identifiable}: In fact, the existence of an anchor word for each $U_j$ automatically implies that $\nbhd(U_i)$ is not a subset of $\nbhd(U_j)$ for  $i\ne j$. Thus, anchor words are a sufficient but not necessary condition for identifiability, whereas Condition~\ref{assm:latentdag} is indeed necessary as described above.
\end{remark}

More details and discussion on these assumptions can be found in Appendix~\ref{sec:latent-dag-setup}.

\subsection{Main identifiability results}
\label{sec:main:results}

When $\dim(U)=1$, there is no additional structure in $U$ to learn, and so the setting simplifies considerably. We begin with this special case before considering the case of general multivariate~$U$.
\begin{thm}\label{thm:main:1d}
Assume $\dim(U)=1$. Under \ref{assm:mixture}-\ref{assm:relu}, we have the following:
\begin{enumerate}[label=(\alph*)]
\item\label{thm:main:1d:invib} \ref{assm:invib}$\implies P(U,Z)$ is identifiable from $P(X)$ up to an affine transformation of $Z$.
\item\label{thm:main:1d:latent} \ref{assm:invib}+\ref{assm:latent}$\implies P(U,Z)$ is identifiable from $P(X)$ up to permutation, scaling, and/or translation of $Z$.
\item\label{thm:main:1d:inv} In either (a) or (b), if additionally \ref{assm:inv} holds and $f$ is continuous, then $f$ is also identifiable from $P(X)$ up to an affine transformation.
\end{enumerate}
\end{thm}
The next result generalizes Theorem~\ref{thm:main:1d} to arbitrary (possibly multivariate) discrete $U$.
This is an especially challenging case: Unlike previous work such as iVAE that assumes $U$ (and hence its structure) is known, we do not assume anything about $U$ is known. Thus, everything about $U$ must be reconstructed based on $P(X)$ alone, hence the need for \ref{assm:latentdag} to identify $P(U)$ below.
\begin{thm}\label{thm:main:highd}
Under \ref{assm:mixture}-\ref{assm:relu}, we have the following:
\begin{enumerate}[label=(\alph*)]
\item\label{thm:main:highd:invib} \ref{assm:invib}$\implies P(Z)$ is identifiable from $P(X)$ up to an affine transformation.
\item\label{thm:main:highd:latent} \ref{assm:invib}+\ref{assm:latent}$\implies P(Z)$ is identifiable from $P(X)$ up to permutation, scaling, and/or translation.
\item\label{thm:main:highd:latentdag} \ref{assm:invib}+\ref{assm:latent}+\ref{assm:latentdag}$\implies (k,d_1,\ldots,d_k,P(U))$ are identifiable from $P(X)$ up to a permutation of $U$, and $P(Z)$ is identifiable up to permutation, scaling, and/or translation.
\item\label{thm:main:highd:inv} In any of (a), (b), or (c), if additionally \ref{assm:inv} holds and $f$ is continuous, then $f$ is also identifiable from $P(X)$ up to an affine transformation.
\end{enumerate}
\end{thm}
Without \ref{assm:latentdag}, \citet{kivva2021learning} have shown that it is not possible to recover the high-dimensional latent state $U$, however, we can still identify the continuous latent state $Z$, which is enough to generate random samples from the model \eqref{eq:defn:nica}. In order to have fine-grained control over the individual variables in $U$, however, it is necessary to assume \ref{assm:latentdag}.

\begin{remark}
\label{rem:A4}
If \ref{assm:inv} is relaxed to \ref{assm:invae} $f$ may not be identifiable up to an affine transformation, but it is ``essentially'' identifiable in the following sense. Let $S = \{x: |f^{-1}(\{x\})|>1\}$. On every connected component of $\mathbb{R}^m\setminus f^{-1}(S)$, $f$ is identifiable up to an affine transformation (which may depend on the connected component). Note, for $f$ defined by a ReLU NN, points of $S$ are atoms of $P(X)$.
\end{remark}

\begin{remark}\label{rem:non-invert-f-example}
If the assumption \ref{assm:invib} that $f$ is weakly injective is removed, then the claim of Theorem~\ref{thm:main:1d} is not true anymore. Consider $g(x) = f(x) = |x|$ and
\begin{equation}
\begin{split}
    & P = \dfrac{1}{3}N(-2, \sigma^2)+\dfrac{1}{3}N(-1, \sigma^2)+\dfrac{1}{3}N(3, \sigma^2) \quad \text{and} \\
    & P' = \dfrac{1}{3}N(-2, \sigma^2)+\dfrac{1}{3}N(1, \sigma^2)+\dfrac{1}{3}N(3, \sigma^2).
\end{split}
\end{equation}
It is easy to verify that $P$ cannot be transformed into $P'$ by an affine transformation, but $f_{\sharp}P$ and $g_{\sharp}P'$ are equally distributed.
\end{remark}

\smallskip
\begin{remark} In Theorems~\ref{thm:main:1d}\ref{thm:main:1d:invib} and \ref{thm:main:highd}\ref{thm:main:highd:invib}, the identifiability up to an affine transformation is the best possible if no additional assumptions on $Z$ are made (i.e. beyond \ref{assm:mixture}).
Indeed, for an arbitrary invertible affine map $h:\mathbb{R}^m \rightarrow \mathbb{R}^m$, $h(Z)$ has a GMM distribution, $f\circ h^{-1}$ is an invertible piecewise affine map, and $(U, Z, f)$ and $(U, h(Z), f\circ h^{-1})$ in model $\eqref{eq:gmm-ivae}$ generate the same distribution.
\end{remark}

\subsection{Special cases}
\label{sec:main:cases}

Our main results contain some notable special cases that warrant additional discussion.

\paragraph{Classical VAE}
The classical, vanilla VAE \citep{kingma2013auto,rezende2014stochastic} with an isotropic Gaussian prior is equivalent to \eqref{eq:gmm-ivae} with $J = 1$. In this case, $U$ is trivial and the Gaussian distribution $P(Z)$ can be transformed by an affine map to a standard isotropic Gaussian $\normalN(0, I)$. In this case, Theorem~\ref{thm:main:1d}\ref{thm:main:1d:inv} shows that $f$ is identifiable from $P(X)$ up to an orthogonal transformation. In fact, this case can readily be deduced from known results on the identifiability of ReLU networks, e.g. \citet{stock2021embedding}. 

Although the $J=1$ case is already identifiable, there are clear reasons to prefer a clustered latent space: It is natural to model data that has several clusters by a latent space that has similar clusters (e.g. Figure~\ref{fig:pinwheel3}). Although in principle any distribution can be approximated by $f(Z)$ where $Z\sim \normalN(0, I)$ and $f$ is piecewise affine, such $f$ is likely to be extremely complex. At the same time, the same distribution may have a representation with $Z$ being a simple GMM and $f$ being a simple piecewise affine function. Clearly, the latter representation is preferable to the former and can likely be more robustly learned in practice. This is consistent with previous empirical work \citep{dilokthanakul2016deep,falck2021multi,jiang2016variational,johnson2016composing,lee2020meta,li2018learning,willetts2019disentangling}.

\paragraph{Linear ICA}
In classical linear ICA \citep{comon1994}, we observe $X=AZ$, where $Z$ is assumed to have independent components. 
Compared to the general model \eqref{eq:defn:nica}, this corresponds to the special case where $f$ is linear and $\eps=0$.
In our most general setting under \ref{assm:invib} only, our results imply that $P(Z)$ can be recovered up to an affine transformation \emph{without} assuming independent components, which might seem surprising at first. 
This is, however, easily explained: In this case, $X$ is also a GMM, and hence $P(Z)$ can already be trivially recovered up to the affine transformation $z\mapsto Az$. This follows from well-known identifiability results for GMMs \citep{teicher1963identifiability}.
This provides some intuition to how the mixture prior assumption \ref{assm:mixture} helps to achieve identifiability.

\paragraph{Nonlinear ICA}
In classical nonlinear ICA, one assumes the model \eqref{eq:defn:nica} with (a) no assumptions on $f$ and (b) independence assumptions in the latent space. It is well-known that this model is nonidentifiable \citep{hyvarinen1999nonlinear}. Our problem setting is distinguished from the classical nonlinear ICA model via assumptions \ref{assm:mixture}-\ref{assm:relu}. While we do not require the $Z_i$ to be mutually independent, we impose assumptions on the form of $f$. It is precisely this inductive bias that allows us to recover identifiability.
As a result, our identifiability theory does not contradict known results such as the Darmois construction \citep{darmois1951analyse} discussed in \citet{hyvarinen1999nonlinear}.

\subsection{Counterexamples}
\label{sec:main:counterexamples}

A natural question is whether or not the mixture prior \ref{assm:mixture} or the piecewise affine nonlinearity \ref{assm:relu} can be relaxed while still maintaining identifiability. In fact, it is not hard to show this is not possible: If either \ref{assm:mixture} or \ref{assm:relu} is broken, then the model \eqref{eq:defn:nica} becomes nonidentifiable. Of course, this is  entirely expected given known negative results on nonlinear ICA \citep{hyvarinen1999nonlinear}. 

\begin{ex}
If $f$ is allowed to be arbitrary, but \ref{assm:mixture} is still enforced, then \eqref{eq:defn:nica} is no longer identifiable: Pick any two GMMs $P=\sum_{j=1}^{J} \lambda_{j}N(\mu_{j},\Sigma_{j})$ and $P'=\sum_{j=1}^{J'} \lambda_{j}'N(\mu_{j}',\Sigma_{j}')$.
Then we can always find a function $g$ such that $g_{\sharp}P'=f_{\sharp}P$  (e.g. use the inverse CDF transform), and $g\ne f$. 
\end{ex}

\begin{ex}
If $P(Z)$ is allowed to be arbitrary, but \ref{assm:relu} is still enforced, then \eqref{eq:defn:nica} is no longer identifiable: Consider any two arbitrary piecewise affine, injective functions $f,g:\R^m\to \R^m$. Then almost surely the preimages $f^{-1}(\{x\})$ and $g^{-1}(\{x\})$ will not be equivalent up to an affine transformation. In other words, fixing $P(X)$, we can find models $(f,P)$ and $(g,P')$ such that $f_{\sharp} P =P(X)=g_{\sharp}P'$, but $f$ is not equivalent to $g$ (i.e. up to any affine transformation).
\end{ex}

\section{Experiments}
\label{sec:expt}

There has been extensive work already to verify empirically that the model \eqref{eq:defn:nica} under \ref{assm:mixture}-\ref{assm:relu} is identifiable. For example, \citet{willetts2021don} observe that deep generative
models with clustered latent spaces are empirically identifiable, and compared this directly to models that rely on side information, and 
\citet{falck2021multi} show that meaningful latent variables can be learned consistently in a fully unsupervised manner even when $U$ has high-dimensional structure. Moreover, \citet{falck2021multi} indicate that high-dimensional structure is important for improved performance. Beyond these, it is well-known that VAEs with mixture priors such as VaDE \citep{jiang2016variational} achieve competitive performance on many benchmark tasks; see \citet{dilokthanakul2016deep,falck2021multi,johnson2016composing,lee2020meta,li2018learning,willetts2019disentangling,lee2020meta} for additional experiments and verification. Building upon the established success of these methods, we augment these experiments as follows: 1) We use simple examples to verify that the likelihood indeed has a unique minimizer at the ground truth parameters; 2) We train VaDE on (misspecified) simulated toy models; and 3) We measure stability (up to affine transformations) of the learnt latent spaces on real data.
To measure this, we report the Mean Correlation Coefficient \cite[Appendix A.2]{khemakhem2020ice} metric, which is standard, and an $L^2$-based alignment metric (denoted by $\dist_{\Aff, L2}$).
Definitions of these metrics and additional details on the experiments can be found in Appendix~\ref{app:expt}.

\paragraph{Maximum likelihood} We simulated models satisfying \ref{assm:mixture}-\ref{assm:relu} by randomly choosing weights and biases for a single-layer ReLU network and randomly generating a GMM with $J=2$ or 3 components. These models are simple enough that exact computation of the MLE along the likelihood surface is feasible via numerical integration (Figure~\ref{fig:mle}). In all our simulations (50 total), the ground truth was the unique minimizer of the negative log-likelihood, as predicted by the theory. 
These examples also illustrate a small-scale test of misspecification in the theoretical model: We include cases where $J$ is misspecified and $f$ fails to satisfy \ref{assm:inv}, but the MLE succeeds anyway.

\begin{figure}
\centering
\includegraphics[width = 0.3\textwidth]{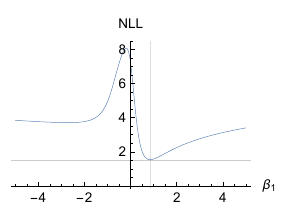}
\includegraphics[width = 0.3\textwidth]{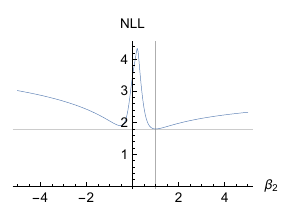}
\includegraphics[width = 0.3\textwidth]{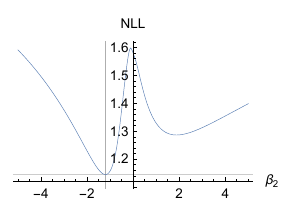}
\caption{Selected examples of the negative log-likelihood for different runs. In each figure, one parameter from a model (e.g. $\beta_j$ is a weight in the neural network defining $f$) is selected, and the value of the negative log-likelihood is visualized as a function of this parameter. Vertical lines indicate the ground truth and (global) minimizer, which always coincide. Three particularly interesting, nonconvex examples are shown here. See Appendix~\ref{app:expt:mle} for details.}\label{fig:mle}
\end{figure}

\paragraph{Simulated data}In our experiments on synthetic datasets we consider, to obtain an experimental evidence of identifiability of model~\eqref{eq:gmm-ivae} we fit VaDE to observed data 5 times (see Figure~\ref{fig:pinwheel3}). Let $Z^{(1)}, Z^{(2)}, \ldots, Z^{(5)}$ be the learned latent spaces. 
For every pair $Z^{(i)}, Z^{(j)}$ we evaluate the MCC and  $\dist_{\Aff, L2}$ loss.
\begin{figure}
\includegraphics[width = \textwidth]{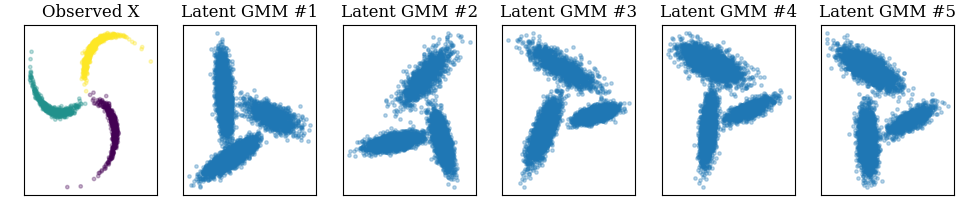}
\caption{Recovered latent spaces for 5 runs of VaDE on pinwheel dataset with 3 clusters}\label{fig:pinwheel3}
\end{figure}
For instance, for the pinwheel dataset with three clusters as in Figure~\ref{fig:pinwheel3}, the average $\dist_{\Aff, L2}(p_1, p_2)$ across 20 pairs $Z^{(i)}, Z^{(j)}$ is 0.113 with standard deviation 0.065. The average weak MCC is 0.87 and the average strong MCC is 1.0. This shows strong evidence of recovery of the latent space up to affine transformations.

\paragraph{Real data} We measure stability of the learnt latent space by training MFCVAE \citep{falck2021multi} on MNIST 10 times with different initializations and then comparing the latent representations learnt. It becomes computationally infeasible to compute $\dist_{\Aff, L2}$ therefore we report only MCC. The strong MCCs are computed to be $0.7$ (ReLU), $0.69$ (LeakyReLU) and the weak MCCs are computed to be $0.91$ (ReLU), $0.94$ (LeakyReLU).
These observations validate the observations first made in \cite{willetts2021don}, who ran extensive experiments on VaDE and iVAE on several large datasets including MNIST, SVHN and CIFAR10.
These strong correlations confirm our theory and are of particular importance to practitioners for whom stability of learning is of the essence.

\section{Conclusion}
\label{sec:conc}

We have proved a general series of results describing a hierarchy of identifiability for deep generative models that are currently used in practice. Our experiments confirm both on exact and approximate simulations that identifiability indeed holds in practice. An obvious direction for future work is to study finite-sample identifiability problems such as sample complexity and robustness (i.e. how many samples are needed to ensure that the global minimizer of the likelihood is reliably close to the ground truth?). Theoretical questions aside, developing a better understanding of the ELBO and its effect on optimization is an important practical question.
For example, an important limitation of the current set of results is that they apply only to the likelihood, which is known to be nonconvex and intractable to optimize (see Figure~\ref{fig:mle} for concrete examples). It is an important open question to use these insights to develop better algorithms and optimization techniques that work on finite-samples with misspecified models (i.e. real data).

More generally, although our assumptions map onto architectures and priors that are widely used in practice, it is important to emphasize the relevant distinction between models and estimators. That is, the architectures used in practice represent the \emph{estimators} used, and may not reflect realistic assumptions on the \emph{model} itself (which is typically misspecified). For example, the piecewise affine assumption may not accurately reflect valid assumptions about real-world problems. Given the lack of purely unsupervised, nonparametric identifiability results in the literature, we view our results as an important technical step towards understanding practical identifiability for deep generative models. Thus, an important future direction is to replace our assumptions with more appropriate modeling assumptions that are relevant for practical applications.

\section{Acknowledgements}

We thank anonymous reviewers for useful comments and suggestions.
G.R. was partially supported
by NSF grants CCF-1816372 and CCF-200892.
B.A. was supported by NSF IIS-1956330, NIH R01GM140467, and the Robert H. Topel Faculty Research Fund at the University of Chicago Booth School of Business. P.R. was supported by ONR via N000141812861, and NSF via IIS-1909816, IIS-1955532, IIS-2211907.

\bibliography{arxiv}

\begin{thebibliography}{75}
\providecommand{\natexlab}[1]{#1}
\providecommand{\url}[1]{\texttt{#1}}
\expandafter\ifx\csname urlstyle\endcsname\relax
  \providecommand{\doi}[1]{doi: #1}\else
  \providecommand{\doi}{doi: \begingroup \urlstyle{rm}\Url}\fi

\bibitem[Achard and Jutten(2005)]{achard2005identifiability}
S.~Achard and C.~Jutten.
\newblock Identifiability of post-nonlinear mixtures.
\newblock \emph{IEEE Signal Processing Letters}, 12\penalty0 (5):\penalty0
  423--426, 2005.

\bibitem[Allman et~al.(2009)Allman, Matias, and Rhodes]{allman2009}
E.~S. Allman, C.~Matias, and J.~A. Rhodes.
\newblock Identifiability of parameters in latent structure models with many
  observed variables.
\newblock \emph{Annals of Statistics}, pages 3099--3132, 2009.

\bibitem[Anandkumar et~al.(2013)Anandkumar, Hsu, Janzamin, and
  Kakade]{anandkumar2013overcomplete}
A.~Anandkumar, D.~J. Hsu, M.~Janzamin, and S.~M. Kakade.
\newblock When are overcomplete topic models identifiable? uniqueness of tensor
  tucker decompositions with structured sparsity.
\newblock \emph{Advances in neural information processing systems}, 26, 2013.

\bibitem[Aragam et~al.(2020)Aragam, Dan, Xing, and Ravikumar]{aragam2018npmix}
B.~Aragam, C.~Dan, E.~P. Xing, and P.~Ravikumar.
\newblock Identifiability of nonparametric mixture models and bayes optimal
  clustering.
\newblock \emph{Ann. Statist.}, 48\penalty0 (4):\penalty0 2277--2302, 2020.
\newblock ISSN 0090-5364.
\newblock \doi{10.1214/19-AOS1887}.
\newblock arXiv:1802.04397.

\bibitem[Arora et~al.(2012)Arora, Ge, and Moitra]{arora2012learning}
S.~Arora, R.~Ge, and A.~Moitra.
\newblock Learning topic models--going beyond svd.
\newblock In \emph{2012 IEEE 53rd annual symposium on foundations of computer
  science}, pages 1--10. IEEE, 2012.

\bibitem[Arora et~al.(2013)Arora, Ge, Halpern, Mimno, Moitra, Sontag, Wu, and
  Zhu]{arora2013practical}
S.~Arora, R.~Ge, Y.~Halpern, D.~Mimno, A.~Moitra, D.~Sontag, Y.~Wu, and M.~Zhu.
\newblock A practical algorithm for topic modeling with provable guarantees.
\newblock In \emph{International Conference on Machine Learning}, pages
  280--288. PMLR, 2013.

\bibitem[Bansal et~al.(2021)Bansal, Nakkiran, and Barak]{bansal2021revisiting}
Y.~Bansal, P.~Nakkiran, and B.~Barak.
\newblock Revisiting model stitching to compare neural representations.
\newblock \emph{Advances in Neural Information Processing Systems}, 34, 2021.

\bibitem[Barndorff-Nielsen(1965)]{barndorff1965}
O.~Barndorff-Nielsen.
\newblock Identifiability of mixtures of exponential families.
\newblock \emph{Journal of Mathematical Analysis and Applications}, 12\penalty0
  (1):\penalty0 115--121, 1965.

\bibitem[Bordes et~al.(2006)Bordes, Mottelet, and Vandekerkhove]{bordes2006}
L.~Bordes, S.~Mottelet, and P.~Vandekerkhove.
\newblock Semiparametric estimation of a two-component mixture model.
\newblock \emph{Annals of Statistics}, 34\penalty0 (3):\penalty0 1204--1232,
  2006.

\bibitem[Brehmer et~al.(2022)Brehmer, De~Haan, Lippe, and
  Cohen]{brehmer2022weakly}
J.~Brehmer, P.~De~Haan, P.~Lippe, and T.~Cohen.
\newblock Weakly supervised causal representation learning.
\newblock \emph{arXiv preprint arXiv:2203.16437}, 2022.

\bibitem[Comon(1994)]{comon1994}
P.~Comon.
\newblock Independent component analysis, a new concept?
\newblock \emph{Signal processing}, 36\penalty0 (3):\penalty0 287--314, 1994.

\bibitem[Csisz{\'a}rik et~al.(2021)Csisz{\'a}rik, K{\H{o}}r{\"o}si-Szab{\'o},
  Matszangosz, Papp, and Varga]{csiszarik2021similarity}
A.~Csisz{\'a}rik, P.~K{\H{o}}r{\"o}si-Szab{\'o}, {\'A}.~Matszangosz, G.~Papp,
  and D.~Varga.
\newblock Similarity and matching of neural network representations.
\newblock \emph{Advances in Neural Information Processing Systems}, 34, 2021.

\bibitem[Dai et~al.(2020)Dai, Wang, and Wipf]{dai2020usual}
B.~Dai, Z.~Wang, and D.~Wipf.
\newblock The usual suspects? reassessing blame for vae posterior collapse.
\newblock In \emph{International Conference on Machine Learning}, pages
  2313--2322. PMLR, 2020.

\bibitem[D'Amour et~al.(2020)D'Amour, Heller, Moldovan, Adlam, Alipanahi,
  Beutel, Chen, Deaton, Eisenstein, Hoffman,
  et~al.]{damour2020underspecification}
A.~D'Amour, K.~Heller, D.~Moldovan, B.~Adlam, B.~Alipanahi, A.~Beutel, C.~Chen,
  J.~Deaton, J.~Eisenstein, M.~D. Hoffman, et~al.
\newblock Underspecification presents challenges for credibility in modern
  machine learning.
\newblock \emph{arXiv preprint arXiv:2011.03395}, 2020.

\bibitem[Darmois(1951)]{darmois1951analyse}
G.~Darmois.
\newblock Analyse des liaisons de probabilit{\'e}.
\newblock In \emph{Proc. Int. Stat. Conferences 1947}, page 231, 1951.

\bibitem[Dilokthanakul et~al.(2016)Dilokthanakul, Mediano, Garnelo, Lee,
  Salimbeni, Arulkumaran, and Shanahan]{dilokthanakul2016deep}
N.~Dilokthanakul, P.~A. Mediano, M.~Garnelo, M.~C. Lee, H.~Salimbeni,
  K.~Arulkumaran, and M.~Shanahan.
\newblock Deep unsupervised clustering with gaussian mixture variational
  autoencoders.
\newblock \emph{arXiv preprint arXiv:1611.02648}, 2016.

\bibitem[Evans(2016)]{evans2016graphs}
R.~J. Evans.
\newblock Graphs for margins of bayesian networks.
\newblock \emph{Scandinavian Journal of Statistics}, 43\penalty0 (3):\penalty0
  625--648, 2016.

\bibitem[Falck et~al.(2021)Falck, Zhang, Willetts, Nicholson, Yau, and
  Holmes]{falck2021multi}
F.~Falck, H.~Zhang, M.~Willetts, G.~Nicholson, C.~Yau, and C.~C. Holmes.
\newblock Multi-facet clustering variational autoencoders.
\newblock \emph{Advances in Neural Information Processing Systems}, 34, 2021.

\bibitem[Gassiat et~al.(2016)Gassiat, Cleynen, and Robin]{gassiat2016inference}
{\'E}.~Gassiat, A.~Cleynen, and S.~Robin.
\newblock Inference in finite state space non parametric hidden markov models
  and applications.
\newblock \emph{Statistics and Computing}, 26\penalty0 (1):\penalty0 61--71,
  2016.

\bibitem[Gresele et~al.(2021)Gresele, Von~K{\"u}gelgen, Stimper, Sch{\"o}lkopf,
  and Besserve]{gresele2021independent}
L.~Gresele, J.~Von~K{\"u}gelgen, V.~Stimper, B.~Sch{\"o}lkopf, and M.~Besserve.
\newblock Independent mechanism analysis, a new concept?
\newblock \emph{Advances in Neural Information Processing Systems}, 34, 2021.

\bibitem[Hall and Zhou(2003)]{hall2003}
P.~Hall and X.-H. Zhou.
\newblock Nonparametric estimation of component distributions in a multivariate
  mixture.
\newblock \emph{Annals of Statistics}, pages 201--224, 2003.

\bibitem[H{\"a}lv{\"a} and Hyvarinen(2020)]{halva2020hidden}
H.~H{\"a}lv{\"a} and A.~Hyvarinen.
\newblock Hidden markov nonlinear ica: Unsupervised learning from nonstationary
  time series.
\newblock In \emph{Conference on Uncertainty in Artificial Intelligence}, pages
  939--948. PMLR, 2020.

\bibitem[H{\"a}lv{\"a} et~al.(2021)H{\"a}lv{\"a}, Corff, Leh{\'e}ricy, So, Zhu,
  Gassiat, and Hyvarinen]{halva2021disentangling}
H.~H{\"a}lv{\"a}, S.~L. Corff, L.~Leh{\'e}ricy, J.~So, Y.~Zhu, E.~Gassiat, and
  A.~Hyvarinen.
\newblock Disentangling identifiable features from noisy data with structured
  nonlinear ica.
\newblock \emph{arXiv preprint arXiv:2106.09620}, 2021.

\bibitem[He et~al.(2018)He, Spokoyny, Neubig, and
  Berg-Kirkpatrick]{he2018lagging}
J.~He, D.~Spokoyny, G.~Neubig, and T.~Berg-Kirkpatrick.
\newblock Lagging inference networks and posterior collapse in variational
  autoencoders.
\newblock In \emph{International Conference on Learning Representations}, 2018.

\bibitem[Hunter et~al.(2007)Hunter, Wang, and Hettmansperger]{hunter2007}
D.~R. Hunter, S.~Wang, and T.~P. Hettmansperger.
\newblock Inference for mixtures of symmetric distributions.
\newblock \emph{Annals of Statistics}, pages 224--251, 2007.

\bibitem[Hyvarinen and Morioka(2016)]{hyvarinen2016unsupervised}
A.~Hyvarinen and H.~Morioka.
\newblock Unsupervised feature extraction by time-contrastive learning and
  nonlinear ica.
\newblock \emph{Advances in Neural Information Processing Systems}, 29, 2016.

\bibitem[Hyvarinen and Morioka(2017)]{hyvarinen2017nonlinear}
A.~Hyvarinen and H.~Morioka.
\newblock Nonlinear ica of temporally dependent stationary sources.
\newblock In \emph{Artificial Intelligence and Statistics}, pages 460--469.
  PMLR, 2017.

\bibitem[Hyv{\"a}rinen and Pajunen(1999)]{hyvarinen1999nonlinear}
A.~Hyv{\"a}rinen and P.~Pajunen.
\newblock Nonlinear independent component analysis: Existence and uniqueness
  results.
\newblock \emph{Neural networks}, 12\penalty0 (3):\penalty0 429--439, 1999.

\bibitem[Hyvarinen et~al.(2019)Hyvarinen, Sasaki, and
  Turner]{hyvarinen2019nonlinear}
A.~Hyvarinen, H.~Sasaki, and R.~Turner.
\newblock Nonlinear ica using auxiliary variables and generalized contrastive
  learning.
\newblock In \emph{The 22nd International Conference on Artificial Intelligence
  and Statistics}, pages 859--868. PMLR, 2019.

\bibitem[Ishikawa et~al.(2022)Ishikawa, Teshima, Tojo, Oono, Ikeda, and
  Sugiyama]{ishikawa2022universal}
I.~Ishikawa, T.~Teshima, K.~Tojo, K.~Oono, M.~Ikeda, and M.~Sugiyama.
\newblock Universal approximation property of invertible neural networks.
\newblock \emph{arXiv preprint arXiv:2204.07415}, 2022.

\bibitem[Iwata et~al.(2013)Iwata, Duvenaud, and Ghahramani]{iwata2013warped}
T.~Iwata, D.~Duvenaud, and Z.~Ghahramani.
\newblock Warped mixtures for nonparametric cluster shapes.
\newblock In \emph{Proceedings of the Twenty-Ninth Conference on Uncertainty in
  Artificial Intelligence}, pages 311--320, 2013.

\bibitem[Jiang et~al.(2016)Jiang, Zheng, Tan, Tang, and
  Zhou]{jiang2016variational}
Z.~Jiang, Y.~Zheng, H.~Tan, B.~Tang, and H.~Zhou.
\newblock Variational deep embedding: An unsupervised and generative approach
  to clustering.
\newblock \emph{arXiv preprint arXiv:1611.05148}, 2016.

\bibitem[Johnson et~al.(2016)Johnson, Duvenaud, Wiltschko, Adams, and
  Datta]{johnson2016composing}
M.~J. Johnson, D.~K. Duvenaud, A.~Wiltschko, R.~P. Adams, and S.~R. Datta.
\newblock Composing graphical models with neural networks for structured
  representations and fast inference.
\newblock \emph{Advances in neural information processing systems}, 29, 2016.

\bibitem[Jutten et~al.(2003)Jutten, Karhunen, et~al.]{jutten2003advances}
C.~Jutten, J.~Karhunen, et~al.
\newblock Advances in nonlinear blind source separation.
\newblock In \emph{Proc. of the 4th Int. Symp. on Independent Component
  Analysis and Blind Signal Separation (ICA2003)}, pages 245--256, 2003.

\bibitem[Khemakhem et~al.(2020{\natexlab{a}})Khemakhem, Kingma, Monti, and
  Hyvarinen]{khemakhem2020variational}
I.~Khemakhem, D.~Kingma, R.~Monti, and A.~Hyvarinen.
\newblock Variational autoencoders and nonlinear ica: A unifying framework.
\newblock In \emph{International Conference on Artificial Intelligence and
  Statistics}, pages 2207--2217. PMLR, 2020{\natexlab{a}}.

\bibitem[Khemakhem et~al.(2020{\natexlab{b}})Khemakhem, Kingma, Monti, and
  Hyv{\"a}rinen]{khemakhem2020ice}
I.~Khemakhem, D.~P. Kingma, R.~P. Monti, and A.~Hyv{\"a}rinen.
\newblock Ice-beem: Identifiable conditional energy-based deep models.
\newblock \emph{NeurIPS2020}, 2020{\natexlab{b}}.

\bibitem[Kingma and Welling(2013)]{kingma2013auto}
D.~P. Kingma and M.~Welling.
\newblock Auto-encoding variational bayes.
\newblock \emph{arXiv preprint arXiv:1312.6114}, 2013.

\bibitem[Kivva et~al.(2021)Kivva, Rajendran, Ravikumar, and
  Aragam]{kivva2021learning}
B.~Kivva, G.~Rajendran, P.~Ravikumar, and B.~Aragam.
\newblock Learning latent causal graphs via mixture oracles.
\newblock \emph{Advances in Neural Information Processing Systems}, 34, 2021.

\bibitem[Klindt et~al.(2020)Klindt, Schott, Sharma, Ustyuzhaninov, Brendel,
  Bethge, and Paiton]{klindt2020towards}
D.~A. Klindt, L.~Schott, Y.~Sharma, I.~Ustyuzhaninov, W.~Brendel, M.~Bethge,
  and D.~Paiton.
\newblock Towards nonlinear disentanglement in natural data with temporal
  sparse coding.
\newblock In \emph{International Conference on Learning Representations}, 2020.

\bibitem[Klys et~al.(2018)Klys, Snell, and Zemel]{klys2018learning}
J.~Klys, J.~Snell, and R.~Zemel.
\newblock Learning latent subspaces in variational autoencoders.
\newblock \emph{Advances in Neural Information Processing Systems}, 31, 2018.

\bibitem[Lee et~al.(2020)Lee, Min, Lee, and Hwang]{lee2020meta}
D.~B. Lee, D.~Min, S.~Lee, and S.~J. Hwang.
\newblock Meta-gmvae: Mixture of gaussian vae for unsupervised meta-learning.
\newblock In \emph{International Conference on Learning Representations}, 2020.

\bibitem[Lenc and Vedaldi(2015)]{lenc2015understanding}
K.~Lenc and A.~Vedaldi.
\newblock Understanding image representations by measuring their equivariance
  and equivalence.
\newblock In \emph{Proceedings of the IEEE conference on computer vision and
  pattern recognition}, pages 991--999, 2015.

\bibitem[Li et~al.(2019)Li, Hooi, and Lee]{li2019identifying}
S.~Li, B.~Hooi, and G.~H. Lee.
\newblock Identifying through flows for recovering latent representations.
\newblock \emph{arXiv preprint arXiv:1909.12555}, 2019.

\bibitem[Li et~al.(2018)Li, Chen, Poon, and Zhang]{li2018learning}
X.~Li, Z.~Chen, L.~K. Poon, and N.~L. Zhang.
\newblock Learning latent superstructures in variational autoencoders for deep
  multidimensional clustering.
\newblock \emph{arXiv preprint arXiv:1803.05206}, 2018.

\bibitem[Locatello et~al.(2019)Locatello, Bauer, Lucic, Raetsch, Gelly,
  Sch{\"o}lkopf, and Bachem]{locatello2019challenging}
F.~Locatello, S.~Bauer, M.~Lucic, G.~Raetsch, S.~Gelly, B.~Sch{\"o}lkopf, and
  O.~Bachem.
\newblock Challenging common assumptions in the unsupervised learning of
  disentangled representations.
\newblock In \emph{international conference on machine learning}, pages
  4114--4124. PMLR, 2019.

\bibitem[Locatello et~al.(2020)Locatello, Poole, R{\"a}tsch, Sch{\"o}lkopf,
  Bachem, and Tschannen]{locatello2020weakly}
F.~Locatello, B.~Poole, G.~R{\"a}tsch, B.~Sch{\"o}lkopf, O.~Bachem, and
  M.~Tschannen.
\newblock Weakly-supervised disentanglement without compromises.
\newblock In \emph{International Conference on Machine Learning}, pages
  6348--6359. PMLR, 2020.

\bibitem[Lu and Lu(2020)]{lu2020universal}
Y.~Lu and J.~Lu.
\newblock A universal approximation theorem of deep neural networks for
  expressing probability distributions.
\newblock \emph{Advances in neural information processing systems},
  33:\penalty0 3094--3105, 2020.

\bibitem[Luise et~al.(2020)Luise, Pontil, and
  Ciliberto]{luise2020generalization}
G.~Luise, M.~Pontil, and C.~Ciliberto.
\newblock Generalization properties of optimal transport gans with latent
  distribution learning.
\newblock \emph{arXiv preprint arXiv:2007.14641}, 2020.

\bibitem[Markham and Grosse-Wentrup(2020)]{markham2020measurement}
A.~Markham and M.~Grosse-Wentrup.
\newblock Measurement dependence inducing latent causal models.
\newblock In \emph{Conference on Uncertainty in Artificial Intelligence}, pages
  590--599. PMLR, 2020.

\bibitem[Mita et~al.(2021)Mita, Filippone, and Michiardi]{pmlr-v139-mita21a}
G.~Mita, M.~Filippone, and P.~Michiardi.
\newblock An identifiable double vae for disentangled representations.
\newblock In M.~Meila and T.~Zhang, editors, \emph{Proceedings of the 38th
  International Conference on Machine Learning}, volume 139 of
  \emph{Proceedings of Machine Learning Research}, pages 7769--7779. PMLR,
  18--24 Jul 2021.

\bibitem[Moran et~al.(2021)Moran, Sridhar, Wang, and
  Blei]{moran2021identifiable}
G.~E. Moran, D.~Sridhar, Y.~Wang, and D.~M. Blei.
\newblock Identifiable variational autoencoders via sparse decoding.
\newblock \emph{arXiv preprint arXiv:2110.10804}, 2021.

\bibitem[Nalisnick et~al.(2016)Nalisnick, Hertel, and
  Smyth]{nalisnick2016approximate}
E.~Nalisnick, L.~Hertel, and P.~Smyth.
\newblock Approximate inference for deep latent gaussian mixtures.
\newblock In \emph{NIPS Workshop on Bayesian Deep Learning}, volume~2, page
  131, 2016.

\bibitem[Nguyen and McLachlan(2019)]{nguyen2019approximations}
H.~D. Nguyen and G.~McLachlan.
\newblock On approximations via convolution-defined mixture models.
\newblock \emph{Communications in Statistics-Theory and Methods}, 48\penalty0
  (16):\penalty0 3945--3955, 2019.

\bibitem[Pearl and Verma(1992)]{pearl1992statistical}
J.~Pearl and T.~S. Verma.
\newblock A statistical semantics for causation.
\newblock \emph{Statistics and Computing}, 2\penalty0 (2):\penalty0 91--95,
  1992.

\bibitem[Ran and Hu(2017)]{ran2017parameter}
Z.-Y. Ran and B.-G. Hu.
\newblock Parameter identifiability in statistical machine learning: a review.
\newblock \emph{Neural Computation}, 29\penalty0 (5):\penalty0 1151--1203,
  2017.

\bibitem[Rezende et~al.(2014)Rezende, Mohamed, and
  Wierstra]{rezende2014stochastic}
D.~J. Rezende, S.~Mohamed, and D.~Wierstra.
\newblock Stochastic backpropagation and approximate inference in deep
  generative models.
\newblock In E.~P. Xing and T.~Jebara, editors, \emph{Proceedings of the 31st
  International Conference on Machine Learning}, volume~32 of \emph{Proceedings
  of Machine Learning Research}, pages 1278--1286, Bejing, China, 22--24 Jun
  2014. PMLR.

\bibitem[Ritchie et~al.(2020)Ritchie, Vandermeulen, and
  Scott]{ritchie2020consistent}
A.~Ritchie, R.~A. Vandermeulen, and C.~Scott.
\newblock Consistent estimation of identifiable nonparametric mixture models
  from grouped observations.
\newblock \emph{arXiv preprint arXiv:2006.07459}, 2020.

\bibitem[Roeder et~al.(2021)Roeder, Metz, and Kingma]{pmlr-v139-roeder21a}
G.~Roeder, L.~Metz, and D.~Kingma.
\newblock On linear identifiability of learned representations.
\newblock In M.~Meila and T.~Zhang, editors, \emph{Proceedings of the 38th
  International Conference on Machine Learning}, volume 139 of
  \emph{Proceedings of Machine Learning Research}, pages 9030--9039. PMLR,
  18--24 Jul 2021.

\bibitem[Sch{\"o}lkopf et~al.(2021)Sch{\"o}lkopf, Locatello, Bauer, Ke,
  Kalchbrenner, Goyal, and Bengio]{scholkopf2021toward}
B.~Sch{\"o}lkopf, F.~Locatello, S.~Bauer, N.~R. Ke, N.~Kalchbrenner, A.~Goyal,
  and Y.~Bengio.
\newblock Toward causal representation learning.
\newblock \emph{Proceedings of the IEEE}, 109\penalty0 (5):\penalty0 612--634,
  2021.

\bibitem[Schott et~al.(2021)Schott, von K{\"u}gelgen, Tr{\"a}uble, Gehler,
  Russell, Bethge, Sch{\"o}lkopf, Locatello, and Brendel]{schott2021visual}
L.~Schott, J.~von K{\"u}gelgen, F.~Tr{\"a}uble, P.~Gehler, C.~Russell,
  M.~Bethge, B.~Sch{\"o}lkopf, F.~Locatello, and W.~Brendel.
\newblock Visual representation learning does not generalize strongly within
  the same domain.
\newblock \emph{arXiv preprint arXiv:2107.08221}, 2021.

\bibitem[Sorrenson et~al.(2019)Sorrenson, Rother, and
  K{\"o}the]{sorrenson2019disentanglement}
P.~Sorrenson, C.~Rother, and U.~K{\"o}the.
\newblock Disentanglement by nonlinear ica with general incompressible-flow
  networks {(GIN)}.
\newblock In \emph{International Conference on Learning Representations}, 2019.

\bibitem[Stock and Gribonval(2021)]{stock2021embedding}
P.~Stock and R.~Gribonval.
\newblock An embedding of relu networks and an analysis of their
  identifiability.
\newblock \emph{arXiv preprint arXiv:2107.09370}, 2021.

\bibitem[Teicher(1963)]{teicher1963identifiability}
H.~Teicher.
\newblock Identifiability of finite mixtures.
\newblock \emph{The annals of Mathematical statistics}, pages 1265--1269, 1963.

\bibitem[Teicher(1967)]{teicher1967}
H.~Teicher.
\newblock Identifiability of mixtures of product measures.
\newblock \emph{The Annals of Mathematical Statistics}, 38\penalty0
  (4):\penalty0 1300--1302, 1967.

\bibitem[Teshima et~al.(2020)Teshima, Ishikawa, Tojo, Oono, Ikeda, and
  Sugiyama]{teshima2020coupling}
T.~Teshima, I.~Ishikawa, K.~Tojo, K.~Oono, M.~Ikeda, and M.~Sugiyama.
\newblock Coupling-based invertible neural networks are universal
  diffeomorphism approximators.
\newblock \emph{Advances in Neural Information Processing Systems},
  33:\penalty0 3362--3373, 2020.

\bibitem[Tomczak and Welling(2018)]{tomczak2018vae}
J.~Tomczak and M.~Welling.
\newblock Vae with a vampprior.
\newblock In \emph{International Conference on Artificial Intelligence and
  Statistics}, pages 1214--1223. PMLR, 2018.

\bibitem[Van Den~Oord et~al.(2017)Van Den~Oord, Vinyals, et~al.]{van2017neural}
A.~Van Den~Oord, O.~Vinyals, et~al.
\newblock Neural discrete representation learning.
\newblock \emph{Advances in neural information processing systems}, 30, 2017.

\bibitem[Vandermeulen et~al.(2019)Vandermeulen, Scott,
  et~al.]{vandermeulen2019operator}
R.~A. Vandermeulen, C.~D. Scott, et~al.
\newblock An operator theoretic approach to nonparametric mixture models.
\newblock \emph{Annals of Statistics}, 47\penalty0 (5):\penalty0 2704--2733,
  2019.

\bibitem[Wang et~al.(2021)Wang, Blei, and Cunningham]{wang2021posterior}
Y.~Wang, D.~Blei, and J.~P. Cunningham.
\newblock Posterior collapse and latent variable non-identifiability.
\newblock \emph{Advances in Neural Information Processing Systems}, 34, 2021.

\bibitem[Willetts and Paige(2021)]{willetts2021don}
M.~Willetts and B.~Paige.
\newblock I don't need $\mathbf{u}$: Identifiable non-linear ica without side
  information.
\newblock \emph{arXiv preprint arXiv:2106.05238}, 2021.

\bibitem[Willetts et~al.(2019)Willetts, Roberts, and
  Holmes]{willetts2019disentangling}
M.~Willetts, S.~Roberts, and C.~Holmes.
\newblock Disentangling to cluster: Gaussian mixture variational ladder
  autoencoders.
\newblock \emph{arXiv preprint arXiv:1909.11501}, 2019.

\bibitem[Yacoby et~al.(2020)Yacoby, Pan, and Doshi-Velez]{yacoby2020failure}
Y.~Yacoby, W.~Pan, and F.~Doshi-Velez.
\newblock Failure modes of variational autoencoders and their effects on
  downstream tasks.
\newblock In \emph{ICML Workshop on Uncertainty and Robustness in Deep Learning
  (UDL)}, 2020.

\bibitem[Yang et~al.(2021)Yang, Wang, Sun, Zhang, Zhang, Li, and
  Yan]{yang2021nonlinear}
X.~Yang, Y.~Wang, J.~Sun, X.~Zhang, S.~Zhang, Z.~Li, and J.~Yan.
\newblock Nonlinear ica using volume-preserving transformations.
\newblock In \emph{International Conference on Learning Representations}, 2021.

\bibitem[Zhang and Chan(2008)]{zhang2008minimal}
K.~Zhang and L.~Chan.
\newblock Minimal nonlinear distortion principle for nonlinear independent
  component analysis.
\newblock \emph{Journal of Machine Learning Research}, 9\penalty0
  (Nov):\penalty0 2455--2487, 2008.

\bibitem[Zimmermann et~al.(2021)Zimmermann, Sharma, Schneider, Bethge, and
  Brendel]{zimmermann2021contrastive}
R.~S. Zimmermann, Y.~Sharma, S.~Schneider, M.~Bethge, and W.~Brendel.
\newblock Contrastive learning inverts the data generating process.
\newblock In \emph{International Conference on Machine Learning}, pages
  12979--12990. PMLR, 2021.

\end{thebibliography}
\bibliographystyle{abbrvnat}

\appendix

\section{Detailed comparisons}
\label{sec:main:compare}

Since the original iVAE paper \citep{khemakhem2020variational}, there have been many generalizations and extensions proposed. We pause here to provide a more detailed comparison of our results against this developing literature. For a comparison against iVAE, see Section~\ref{sec:prelim}.

We first discuss related work that assumes auxiliary information is available (i.e. $U$ is known), then discuss more recent work that does not assume any auxiliary information; the ensuing comparisons are then presented in alphabetical order.

\noindent
\paragraph{Assuming auxiliary information is available.}
\begin{enumerate}
    \item \citet{halva2020hidden} achieves identifiability in the fully unsupervised regime for the model in which the latent state is defined by a  Hidden Markov Model (HMM). 
The proof of identifiability in \citet{halva2020hidden} invokes \citet{gassiat2016inference} to essentially recover the HMM transition matrix and the auxiliary variable $U$ from $X$, reducing the problem to \citet{khemakhem2020variational}. Our Theorem~\ref{thm:dontneedu} shows that identifiability in fully unsupervised regime is possible even without additional structure given here by the time-dependency according to
    Markov dynamics.

    \item \citet{khemakhem2020ice} extend \citet{khemakhem2020variational} by observing that the conditional independence $Z_i \indep Z_j \mid U$ is not required for identifiability, so they propose a more general IMCA framework for conditional energy-based models. However, identifiability in \citet{khemakhem2020ice} still critically relies on observing an auxiliary variable (in their setting, this is a dependent variable $Y$). Our Theorem~\ref{thm:dontneedu} achieves same type of identifiability as \citet{khemakhem2020ice} (up to affine transformation) without relying on conditional independence or an auxiliary variable.

    \item \citet{sorrenson2019disentanglement} extends the iVAE identifiability theory of \citet{khemakhem2020variational} by showing that a stronger notion of identifiability can be achieved if $Z$ is distributed according to factorial GMM (instead of a general exponential family as in \citealp{khemakhem2020variational}). 
More specifically, given the auxiliary information $U$, they show that $Z$ can be recovered up to permutation and scaling of the variables $Z_i$. By contrast, in Theorem~\ref{thm:ivae-up-to-perm}, we show that under similar assumptions $Z_i$ are identifiable up to permutation and scaling and importantly, we do this only from $X$, without using $U$ in any way. 
We also do not require the GMM to be factorial.
    Finally, our proof technique is different: While \citet{sorrenson2019disentanglement} relies on \citet{khemakhem2020variational} (and hence, for instance, require $\ncomp\geq m+1$), our proof is independent of \citet{khemakhem2020variational}.

    \item \citet{yang2021nonlinear} studies identifiability of the model~\eqref{eq:gmm-ivae} under the assumption that $f$ is volume preserving and $Z$ comes from a conditionally factorial exponential family, similar to iVAE. They prove that if $U$ is known, \ref{assm:latent} holds, and $f$ is twice differentiable, then $Z$ is identifiable up to permutation and non-linear functions applied to each $Z_i$ (i.e., $Z_i = h_i(Z_{\tau(i)}$). If additionally $Z$ is a GMM, then $Z$ can be recovered up to permutation, scaling, and translation.
In comparison, we do not require $U$ to be known, and we do not require $f$ to be volume preserving or even differentiable everywhere. We show that under the same assumption \ref{assm:latent} the latent variables $Z_i$ can be recovered up to permutation, scaling, and translation if $f$ is only assumed to be piecewise affine. Additionally, we show that a weaker notion of identifiability holds if $Z$ is not assumed to be conditionally factorial.

    \item \cite{zimmermann2021contrastive} considers a contrastive model in which samples arrive in pairs, which is a type of weak supervision. Additionally, it is assumed that the latent variables are sampled uniformly from a convex body, and that $f$ is differentiable and injective. By comparison, our model allows for more general non-uniform mixture priors, non-injective and non-smooth $f$, and is fully unsupervised.

\end{enumerate}

\noindent
\paragraph{No auxiliary information.}
\begin{enumerate}

    \item \citet{falck2021multi} propose a novel Multifacet VAE (MFCVAE) model for unsupervised deep clustering. Their model has the following form
    \begin{equation}\label{eq:multifacet}
     p(x, z, u) = p(x|z)\prod_{j=1}^{k}p(z_j|u_j)p(u_j), \quad z_j|u_j \sim \normalN(\mu_{u_j}, \Sigma_{u_j})
    \end{equation}
    Through empirical experiments, \citet{falck2021multi} emphasizes the importance of high-dimentional structure of $U$ and shows how it results in improved clustering performance. The key idea is that while the number of meaningful clusters in the data may be very large, there may be meaningful individual categorical variables $U_i$ (``facets'') with a much smaller number of states, which may be easier to learn. In this way, by simultaneously performing clustering for each ``facet" $U_i$ one can learn meaningful fine-grained clusters in the data. Note that $k$ binary variables $U_i$ result in $\ncomp = 2^{k}$ fine-grained clusters in the data.   

    Compared to our work, \citet{falck2021multi} is focused on practical implementation details, and lacks a formal identifiability theory.
    In fact, our results provide precisely such a formal identifiability theory in a more general setting. If $p(x|z)$ is modeled by ReLU/leaky-ReLU NN, MFCVAE is a special case of our model~\eqref{eq:gmm-ivae} with high-dimensional $U$.  More specifically, the MFCVAE model \eqref{eq:multifacet} restricts our model \eqref{eq:gmm-ivae} to the case when $u_i$ are independent and $\nbhd(U_i) = \{i\}$. In particular, it satisfies assumption \ref{assm:latentdag}. Therefore, Theorem~\ref{thm:main:highd} implies that for MFCVAE with diagonal covariances $\Sigma_{u_j}$, $\dim(U), \dim(U_j)$, $P(U)$ are identifiable from $P(X)$ up to a permutation of $U$, and $P(Z)$ is identifiable up to permutation, scaling, and/or translation.

    \item \citet{kivva2021learning} establishes the identifiability of latent representations for non-parametric measurement models $U\rightarrow X$. Their result crucially relies on the fact that observed variables are conditionally independent $X_i\indep X_j\mid U$. Our Theorem~\ref{thm:ivae-latentdag} significantly generalizes this result, by showing the same guarantees for the model~\eqref{eq:gmm-ivae} that allows arbitrarily complex dependencies between the observed variables $X$.

    \item \citet{moran2021identifiable} propose a sparse VAE and prove that the latent space of this model is identifiable. Similar to \citet{wang2021posterior}, identifiability of $f$ is not addressed. Their identifiability results also assume an anchor feature assumption, which we do not require. Even our strongest assumption \ref{assm:latentdag} is weaker compared to the anchor feature assumption (see Remark~\ref{rem:anchor}). Moreover, we do not require any sparsity assumptions.

    \item \citet{wang2021posterior} propose LIDVAE as a way to identify the latent space of a VAE without auxiliary information, however, their approach only guararantees identifiability of $P(Z)$, and does not address $f$ (this is acknowledged by the authors in their discussion as an open question).
By restricting $f$ to be a Brenier map, they guarantee that the likelihood is injective, which leads to identifiability of $P(Z)$. Compared to \cite{wang2021posterior} our work restricts $f$ in a different way (i.e. by an injective ReLU network), which matches common practice. Moreover, we show that both $f$ and the multivariate $U$ structure (i.e. in addition to $P(Z)$) are identifiable under mild additional assumptions.

\end{enumerate}

\section{Proof outline}\label{app:proof}

We will prove the main results by breaking the argument into four phases:
\begin{enumerate}
\item (Appendix~\ref{sec:proofdontneedu}) First, we show that if $f$ is weakly injective, then $\pr(Z)$ is identifiable (Theorem~\ref{thm:dontneedu}). The proof involves a novel result on identifiability of a nonparametric mixture model (Theorem~\ref{thm:main:npmixF2}) that may be of independent interest.
\item (Appendix~\ref{sec:identf}) Second, we show that if $f$ is continuous and injective, then $f$ is identifiable up to an affine transformation (Theorem~\ref{thm:dontneedu-strong}). This result strengthens existing identifiability results in nonlinear ICA by exploiting the mixture prior, which is crucial in the sequel.
\item (Appendix~\ref{sec:icagmmproof}) Next, we show that if $Z$ is conditionally factorial GMM, then under mild generic assumptions, the individual variables $Z_i$ can be recovered (up to permutation, scaling and translation) (Theorem~\ref{thm:gmm-ICA}). 
\item (Appendix~\ref{sec:latent-dag-setup}) Finally, since for conditionally factorial $Z$ we are able to recover the individual variables $Z_i$, we show how we can apply the theory developed in~\cite{kivva2021learning} to recover the multivariate discrete latent variable $U$, its dimension, domain sizes of each $U_i$ and $\Pr(U, Z)$ (Theorem~\ref{thm:ivae-latentdag}). Since we can only recover $Z$ up to permutation, scaling and translation, the results from~\cite{kivva2021learning} cannot be applied directly, and we show how to perform this recovery under an unknown affine transformation.
\end{enumerate}

Each of these phases tackles a particular level of the identifiability hierarchy described in the main theorems. A detailed proof outline of each main theorem is provided below; technical proofs can be found in the subsequent appendices.

A notable difference between Theorems~\ref{thm:main:1d} ($k=1$) and~\ref{thm:main:highd} ($k>1$) is the conclusion in the latent space: Theorem~\ref{thm:main:1d} identifies $P(U,Z)$ jointly whereas Theorem~\ref{thm:main:highd} identifies $P(Z)$ and $P(U)$ separately. The reason is simple: 
If $U$ is 1-dimensional, i.e., $k=1$, then $P(U, Z)$ for~\eqref{eq:gmm-ivae} is trivially identifiable from $P(Z)$, since $P(Z)$ is assumed to be a GMM by \ref{assm:mixture}. Indeed, since finite mixture of Gaussians are identifiable, we can recover $P(U = u)$ and $P(Z\mid U = u)$ as mixture weights and corresponding Gaussian components. This extends to more general exponential mixtures as in Remark~\ref{rem:gmm}, see \citet{barndorff1965} for details.

When $k>1$, the situation is considerably more nontrivial, as one also needs to learn the high-dimensional structure of $U$.

\begin{proof}[Proof of Theorem~\ref{thm:main:1d}]
We assume $\varepsilon=0$ without any loss of generality; i.e. it is sufficient to consider the noiseless case.
This follows from a standard deconvolution argument as in \citet{khemakhem2020ice} (see Step I of the proof of Theorem~1).
\begin{enumerate}[label=(\alph*)]
    \item By Theorem~\ref{thm:dontneedu}, $\pr(Z)$ is identifiable up to an affine transformation. Moreover, as described above, we can identify $\pr(U, Z)$ from $\pr(Z)$. 
\item Since $P(Z)$ is identifiable up to an affine transformation by part a), claim follows from Theorem~\ref{thm:gmm-ICA}.
    \item By Theorem~\ref{thm:dontneedu-strong}, $f$ is identifiable. \qedhere
\end{enumerate}
\end{proof}

\begin{proof}[Proof of Theorem~\ref{thm:main:highd}]
As with Theorem~\ref{thm:main:1d}, we assume $\varepsilon=0$ without loss of generality.
\begin{enumerate}[label=(\alph*)]
    \item By Theorem~\ref{thm:dontneedu}, $\pr(Z)$ is identifiable up to an affine transformation.
    \item  Since $P(Z)$ is identifiable up to an affine transformation by part a), by Theorem~\ref{thm:gmm-ICA}, $Z_i$ are identifiable up to permutation, scaling and translation.
    \item Follows from Theorem~\ref{thm:ivae-latentdag}.
    \item By Theorem~\ref{thm:dontneedu-strong}, $f$ is identifiable. \qedhere
\end{enumerate}
\end{proof}

\section{Identifiability of $Z$ up to an affine transformation via nonparametric mixtures}\label{sec:proofdontneedu}

In this section we prove that if in model~\eqref{eq:gmm-ivae} the function $f$ is weakly injective, then $Z$ is identifiable up to an affine transformation. More specifically, we prove the following:

\begin{theorem}\label{thm:dontneedu}
Assume that $(U, Z, X)$ are distributed according to model $\eqref{eq:gmm-ivae}$. If $f$ is weakly injective (see \ref{assm:invib} in Definition~\ref{def:invertible}), then $\pr(U, Z)$ is identifiable from $\pr(X)$ up to an affine transformation.
\end{theorem}

We will prove this result by first proving a result on identifiability of nonparametric mixtures that may be of independent interest.

\begin{theorem}\label{thm:main:npmixF2}
Let $f, g:\mathbb{R}^m \rightarrow \mathbb{R}^n$ be piecewise affine functions satisfying \ref{assm:invib}. Let $Y~\sim \sum\limits_{i=1}^{J} \lambda_i\normalN(\mu_i, \Sigma_i)$ and $Y' \sim \sum\limits_{j=1}^{J'} \lambda_j'\normalN(\mu_j', \Sigma_j')$ be a pair of GMMs (in reduced form). Suppose that $f(Y)$ and $g(Y')$ are equally distributed.

 Then there exists an invertible affine transformation $h:\mathbb{R}^m \rightarrow \mathbb{R}^m$ such that $h(Y) \equiv Y'$, i.e., $J = J'$ and for some permutation $\tau\in S_J$ we have $\lambda_i = \lambda'_{\tau(i)}$ and $h_{\sharp}\normalN(\mu_i, \Sigma_i) = \normalN(\mu'_{\tau(i)}, \Sigma'_{\tau(i)})$.
\end{theorem}

In other words, a mixture model whose components are piecewise affine transformations of a Gaussian is identifiable. To see this more clearly, observe that 
\begin{align*}
    \sum_{j=1}^{J} \lambda_{k}f_{\sharp}\normalN(\mu_{k},\sigma_{k})
    \sim
    f_{\sharp}\Big(\sum_{j=1}^{J} \lambda_{k}\normalN(\mu_{k},\sigma_{k})\Big).
\end{align*}
To the best of our knowledge, this identifiability result for a nonparametric mixture model is new to the literature.
In Theorem~\ref{thm:main:npmixF2}, the transformation and number of components is allowed to be unknown and arbitrary, and no separation or independence assumptions are needed.

\subsection{Technical lemmas}\label{sec:technical-lemas-iden}

We recall that a $m$-dimensional Gaussian distribution $\normalN(\mu, \Sigma)$ with covariance $\Sigma$ and mean $\mu$ has the following density function
\begin{equation}
    p(x) = \dfrac{1}{\sqrt{(2\pi)^m\det \Sigma}}\exp\left((-1/2)(x-\mu)^T\Sigma^{-1}(x-\mu)\right).
\end{equation}
We assume that all Gaussian components are \emph{non-degenerate} in the sense that $\Sigma$ is positive definite. 
We also recall that if $Y\sim \normalN(\mu, \Sigma)$ and $Y' = AY+b$ for an invertible $A\in \mathbb{R}^{m\times m}$ and $b\in \mathbb{R}^m$, then $Y'\sim \normalN(A\mu+b, A\Sigma A^T) $.

\begin{definition}
We say that a Gaussian mixture distribution
\begin{equation}
    P = \sum\limits_{j = 1}^J \lambda_j \normalN(\mu_j, \Sigma_j)
\end{equation}
is in reduced form if $\lambda_j>0$ for every $j\in [J]$ and for every $i\neq j\in [J]$ we have $(\mu_i, \Sigma_i)\neq (\mu_j, \Sigma_j)$.
\end{definition}

In the proofs we use the notion of real analytic functions. We remind the definition for reader's convenience.
\begin{definition}
Let $D\subseteq \mathbb{R}^n$ be an open set. A function $f:D \rightarrow \mathbb{R}$ is called a (real) analytic function if for every compact $K\subset D$ there exists a constant $C>0$ such that for any $\alpha \in \mathbb{N}^{n}$ we have
\begin{equation}
    \sup\limits_{x\in K} \left| \dfrac{\partial^{\alpha} f}{\partial x^{\alpha}}  (x)\right|\leq \alpha!C^{|\alpha|+1}.
\end{equation}
Alternatively, a real analytic function $f:D\rightarrow \mathbb{R}$ can be defined as a function that has a Taylor expansion convergent on $D$. 
\end{definition}

It is a standard fact that a linear combination and a product of analytic functions are analytic, and it is well-known that the density of the multivariate Gaussian is a real analytic function on $\mathbb{R}^m$. We will also need the standard notion of analytic continuation:
\begin{definition}
\label{defn:acont}
Let $D_0\subseteq D\subseteq  \mathbb{R}^n$ be open sets. Let $f_0:D_0\rightarrow \mathbb{R}$. We say that an analytic function $f:D\rightarrow \mathbb{R}$ is an \emph{analytic continuation} of $f_0$ onto $D$ if $f(x) = f_0(x)$ for every $x\in D_0$.
\end{definition}

\begin{definition}
Let $x_0\in \mathbb{R}^m$ and $\delta>0$. Let $p:B(x_0, \delta)\rightarrow \mathbb{R}$.  Define
\begin{equation}
    \Ext(p): \mathbb{R}^m\rightarrow \mathbb{R}
\end{equation}
to be the unique analytic continuation of $p$ on the entire space $\mathbb{R}^m$ if such a continuation exists, and to be $0$ otherwise.
\end{definition}

\begin{definition}
Let $D_0\subset D$ and $p:D\rightarrow \mathbb{R}$ be a function. We define $p|_{D_0}: D_0\rightarrow \mathbb{R}$ to be a restriction of $p$ to $D_0$, namely a function that satisfies $p|_{D_0}(x) = p(x)$ for every $x\in D_0$.
\end{definition}

\begin{theorem}\label{thm:local-gmm-iden}
Consider a pair of finite GMMs (in reduced form) in $\mathbb{R}^m$
\begin{equation}
    P=\sum_{j=1}^{J} \lambda_{j}\normalN(\mu_{j}, \Sigma_{j})\quad \text{and}\quad P'=\sum_{j=1}^{J'} \lambda_{j}'\normalN(\mu'_{j}, \Sigma'_{j}).
\end{equation}
Assume that there exists a ball $B(x_0, \delta)$ such that $P$ and $P'$ induce the same measure on $B(x_0, \delta)$. Then $P \equiv P'$, i.e., $J = J'$ and for some permutation $\tau$ we have $\lambda_i = \lambda'_{\tau(i)}$ and $(\mu_i, \Sigma_i) = (\mu'_{\tau(i)}, \Sigma'_{\tau(i)})$.
\end{theorem}
\begin{proof}
Follows from the identity theorem for real analytic functions and the identifiability of finite GMMs.
\end{proof}

\begin{definition}
Let $f:\mathbb{R}^m\rightarrow \mathbb{R}^n$ be a piecewise affine function. We say that a point $x\in f(\mathbb{R}^m)\subseteq \mathbb{R}^n$ is \emph{generic with respect to $f$} if the preimage $f^{-1}(\{x\})$ is finite and there exists $\delta>0$, such that $f:B(z, \delta)\rightarrow \mathbb{R}^n$ is affine for every $z\in f^{-1}(\{x\})$. 
\end{definition}

\begin{lemma}\label{obs:oa-generic}
If $f:\mathbb{R}^m\rightarrow \mathbb{R}^n$ is a piecewise affine function such that $\{x\in \mathbb{R}^n\, :\, |f^{-1}(\{x\})| = \infty\}\subseteq f(\mathbb{R}^m)$ has measure zero with respect to the Lebesgue measure on $f(\mathbb{R}^m)$, then $\dim (f(\mathbb{R}^m)) = m$ and almost every point in $f(\mathbb{R}^m)$ (with respect to the Lebesgue measure on $f(\mathbb{R}^m)$) is generic with respect to~$f$.
\end{lemma}
\begin{proof}
Let $g_i(z) = Az+b$, $g:D\rightarrow \mathbb{R}^n$ be one of the affine pieces defining piecewise affine function $f$. If $A$ does not have full column rank, then every $x\in g(D)$ has an infinite number of preimages. Therefore, the assumption of the lemma implies that for at least one of the affine pieces $g_i$, $A$ has full column rank. Thus, $\dim (f(\mathbb{R}^m)) = m$.

Let $S = \{x\in \mathbb{R}^n\, :\, |f^{-1}(\{x\})| = \infty\}$ then by assumption $S$ has measure zero in $f(\mathbb{R}^m)$. Let $E$ be the set of points $z\in \mathbb{R}^m$ such that for every $\delta>0$, $f$ is not affine on $B(z, \delta)$.  
Since $f$ is piecewise affine, $E$ can be covered by a locally-finite union of $(m-1)$-dimensional subspaces, i.e. every compact set intersects only finitely many of these (potentially infinite) $(m-1)$-dimensional subspaces. Thus $E$ has measure zero. Moreover, since $\dim (f(\mathbb{R}^m)) = m$, $f(E)$ has measure zero in $f(\mathbb{R}^m)$. 

Finally, by definition, every $x\in f(\mathbb{R}^m)\setminus \left(S \cup f(E)\right)$ is generic.
\end{proof}

We make the following useful observation.

\begin{lemma}\label{lem:preimage-size}
Consider a random variable $Z$ distributed according to the GMM $ \sum\limits_{j =1}^{J}\lambda_{j}\normalN(\mu_j, \Sigma_j)$. Consider the random variable $X = f(Z)$, where $f:\mathbb{R}^m\rightarrow \mathbb{R}^m$ is a piecewise affine function, such that $\dim(f(\mathbb{R}^m)) = m$. Let $x_0\in \mathbb{R}^m$ be a generic point with respect to $f$. Let $p$ be the density function of $X$. Then the number of points in the preimage $f^{-1}(\{x_0\})$ can be computed as
\begin{equation}
    |f^{-1}(\{x_0\})| = \lim\limits_{\delta \rightarrow 0} \int_{x\in \mathbb{R}^m} \Ext(p|_{B(x_0, \delta)})(x)dx.
\end{equation}
\end{lemma}
\begin{proof}
Since $x_0$ is generic with respect to $f$, the preimage of $x_0$ consists of finitely many points, $f^{-1}(\{x_0\}) = \{z_1, z_2, \ldots, z_s\}$, and there exists $\varepsilon>0$ such that for every $i\in [s]$ there is a well-defined invertible affine function $g_i:B(z_i,\varepsilon)\rightarrow \mathbb{R}^m$ such that $g_i(z) = f(z)$ for all $z\in B(z_i,\varepsilon)$.

We can write $g_i(z) = A_iz+b_i$ for some $A_i\in \mathbb{R}^{m\times m}$ and $b_i\in \mathbb{R}^m$. Let $\delta_0>0$ be such that
\begin{equation}
    B(x_0, \delta_0)\subseteq \bigcap_{i=1}^{s} g_i(B(z_i, \varepsilon)).
\end{equation}
Let $0<\delta<\delta_0$. Then, for $\mu_{ij}' = A_i\mu_k+b_i$ and $\Sigma_{ij} = A_i\Sigma_j A_i^T$, and every $x\in B(x_0, \delta)$ we have
\begin{equation}\label{eq:local-mixture-image}
p|_{B(x_0, \delta)}(x) = \sum\limits_{i =1}^{s}\sum\limits_{j = 1}^{J} \dfrac{\lambda_j}{\sqrt{(2\pi)^m\det \Sigma_{ij}}}\exp\left((-1/2)(x-\mu_{ij}')^T\Sigma_{ij}^{-1}(x-\mu_{ij}')\right).
\end{equation}
The RHS of \eqref{eq:local-mixture-image} is a real analytic function defined on all of $\mathbb{R}^m$ (i.e. it is an entire function) that equals $p$ on an open neighborhood, hence it defines $\Ext(p|_{B(x_0, \delta)})$ on the entire space $\mathbb{R}^m$.
Therefore,
\begin{equation}
\begin{split}
    & \int_{x\in \mathbb{R}^m} \Ext(p|_{B(x_0, \delta)})(x)dx = \\
    & = \int_{x\in \mathbb{R}^m} \sum\limits_{i =1}^{s}\sum\limits_{j = 1}^{J} \dfrac{\lambda_j}{\sqrt{(2\pi)^m\det \Sigma_{ij}}}\exp\left((-1/2)(x-\mu_{ij}')^T\Sigma_{ij}^{-1}(x-\mu_{ij}')\right) = \\
     & = \sum\limits_{i =1}^{s}\int_{x\in \mathbb{R}^m} \sum\limits_{j = 1}^{J} \dfrac{\lambda_j}{\sqrt{(2\pi)^m\det \Sigma_{ij}}}\exp\left((-1/2)(x-\mu_{ij}')^T\Sigma_{ij}^{-1}(x-\mu_{ij}')\right) = \\
      & = s = |f^{-1}(\{x_0\})|. \qedhere
\end{split}
\end{equation}
\end{proof}

We can deduce the following corollary.

\begin{corollary}\label{cor:invertible_fg}
Let $f, g:\mathbb{R}^m \rightarrow \mathbb{R}^n$ be piecewise affine functions that satisfy~\ref{assm:invib}. 

Let $Z~\sim \sum\limits_{i=1}^{J} \lambda_i\normalN(\mu_i, \Sigma_i)$ and $Z' \sim \sum\limits_{j=1}^{J'} \lambda_j'\normalN(\mu_j', \Sigma_j')$. Suppose that $f(Z)$ and $g(Z')$ are equally distributed. Assume that for $x_0\in \mathbb{R}^n$ and $\delta>0$, $f$ is invertible on $B(x_0, 2\delta)\cap f(\mathbb{R}^{m})$.

Then there exists $x_1\in B(x_0, \delta)$ and $\delta_1>0$ such that both $f$ and $g$ are invertible on $B(x_1, \delta_1)\cap f(\mathbb{R}^{m})$. 
\end{corollary}
\begin{proof}
Since $f$ is piecewise affine and $f$ is invertible on $B(x_0, 2\delta)\cap f(\mathbb{R}^m)$, then $\dim f(\mathbb{R}^m) = m$.
Note that since $f(Z)$ and $g(Z')$ are equally distributed and since regular GMMs have positive density at every point, we have
\[ f(\mathbb{R}^m) = \supp(f(Z)) = \supp(g(Z')) = g(\mathbb{R}^m). \]
Therefore, $\dim(g(\mathbb{R}^m)) = \dim(f(\mathbb{R}^m)) = m$ and, by Lemma~\ref{obs:oa-generic}, almost every point $x\in B(x_0, \delta) \cap f(\mathbb{R}^m) $ is generic with respect to $f$ and w.r.t to $g$. Let $x_1 \in B(x_0, \delta)$ be such a point. Since $f$ is invertible on $B(x_1, \delta)$, we have that $|f^{-1}(\{x_1\})| = 1$. Since $x_1$ is generic with respect to $f$ and with respect to to $g$, by Lemma~\ref{lem:preimage-size}, we deduce that $|g^{-1}(\{x_1\})| = 1$. Therefore, since $x_1$ is generic, there exists $0<\delta_1<\delta$ such that on  $\left(B(x_1, \delta_1) \cap f(\mathbb{R}^m)\right) \subset \left(B(x_0, 2\delta)\cap f(\mathbb{R}^m)\right)$ the function $g$ is invertible.
\end{proof}

\subsection{Identifiability of nonparametric mixtures}

First we prove our identifiability theorem under the assumption that $f$ and $g$ are invertible in the neighborhood of the same point.

\begin{thm} \label{thm:identif-inv-affine}
Let $f, g:\mathbb{R}^m \rightarrow \mathbb{R}^n$ be piecewise affine. Let $Z~\sim \sum\limits_{i=1}^{J} \lambda_i\normalN(\mu_i, \Sigma_i)$ and $Z' \sim \sum\limits_{j=1}^{J'} \lambda_j'\normalN(\mu_j', \Sigma_j')$ be a pair of GMMs (in reduced form). Suppose that $f(Z)$ and $g(Z')$ are equally distributed.

Assume that there exists $x_0\in \mathbb{R}^n$ and $\delta>0$ such that $f$ and $g$ are invertible on $B(x_0, \delta)\cap f(\mathbb{R}^m)$. Then there exists an invertible affine transformation $h:\mathbb{R}^m \rightarrow \mathbb{R}^m$ such that $h(Z) \equiv Z'$, i.e., $J = J'$ and for some permutation $\tau$ we have $\lambda_i = \lambda'_{\tau(i)}$ and $h_{\sharp}\normalN(\mu_i, \Sigma_i) = \normalN(\mu'_{\tau(i)}, \Sigma'_{\tau(i)})$.
\end{thm}
\begin{proof} Since $f$ and $g$ are piecewise affine and both $f$ and $g$ are invertible on $B(x_0, \delta)\cap f(\mathbb{R}^m)$, then $\dim f(\mathbb{R}^m) = m$ and the inverse functions are piecewise affine. Hence, moreover, there exist $x_1$ and $\delta_1>0$ with $B(x_1, \delta_1)\subseteq B(x_0, \delta)$ such that $f^{-1}$ and $g^{-1}$ on $B(x_1, \delta_1)\subseteq B(x_0, \delta)$ are defined by affine functions.

Let $L\subseteq \mathbb{R}^n$ be an $m$-dimensional affine subspace, such that $B(x_1, \delta_1)\cap f(\mathbb{R}^m)  = B(x_1, \delta_1)\cap L$.

Let $h_f, h_g: \mathbb{R}^m\rightarrow L$ be a pair of invertible affine functions such that $h_f^{-1}$ coincides with $f^{-1}$ on $B(x_1, \delta_1)\cap L$ and $h_g^{-1}$ coincides with $g^{-1}$ on $B(x_1, \delta_1)\cap L$. This means that distributions $h_f(Y)$ and $h_g(Y')$ coincide on $B(x_1, \delta_1)\cap L$. Moreover, since $h_f$ and $h_g$ are affine transformations, then $h_f(Y)$ and $h_g(Y')$ are finite GMMs. Therefore, by Theorem~\ref{thm:local-gmm-iden}, $h_f(Y)\equiv h_g(Y')$. The claim of the theorem holds for $h = h_g^{-1}\circ h_f$.
\end{proof}

Combining this identifiability result with results of Section~\ref{sec:technical-lemas-iden}, we obtain the proof of our main identifiability result for non-parametric mixtures.

\begin{proof}[Proof of Theorem~\ref{thm:main:npmixF2}]
By Corollary~\ref{cor:invertible_fg} there exists $x_0\in f(\mathbb{R}^m)$ that is generic with respect to to both $f$ and $g$ and $\delta>0$ such that $f$ and $g$ are invertible on $B(x_0, \delta)\cap f(\mathbb{R}^m)$. Therefore, the result follows from Theorem~\ref{thm:identif-inv-affine}.
\end{proof}

\subsection{Proof of Theorem~\ref{thm:dontneedu}}

We give a proof by contradiction. Assume that there exists another model $(U', Z', X')$ and a piecewise affine function $g$ in model~\ref{eq:gmm-ivae} that generates the same distribution, i.e., $\pr(X) = \pr(X')$.

By Corollary~\ref{cor:invertible_fg} there exists $x_0\in f(\mathbb{R}^m)$ that is generic with respect to to both $f$ and $g$ and $\delta>0$ such that $f$ and $g$ are invertible on $B(x_0, \delta)\cap f(\mathbb{R}^m)$. Therefore, by Theorem~\ref{thm:identif-inv-affine}, there exists $h:\mathbb{R}^m\rightarrow \mathbb{R}^m$ such that $Z' = h(Z)$. In other words, $P(U, Z)$ is identifiable up to an affine transformation.

\section{Identifiability of $f$}
\label{sec:identf}

In this section we show that if $f$ is continuous piecewise affine and injective then it is identifiable from $P(X)$ up to an affine transformation. 

\begin{theorem}\label{thm:dontneedu-strong}
Assume that $(U, Z, X)$ are distributed according to model $\eqref{eq:gmm-ivae}$. Assume that $f$ is continuous piecewise affine and satisfies \textup{\ref{assm:inv}} (i.e., $f$ is injective).

Then $(\pr(U, Z), f)$ is identifiable from $\pr(X)$ up to an affine transformation.
\end{theorem}

Before proving this theorem, we provide an example that shows that assumption~\ref{assm:invib} does not guarantee that $f$ can be recovered uniquely up to an affine transformation in Theorem~\ref{thm:dontneedu}.

\begin{ex}\label{ex:invertible-ib-fail}
 Consider
\begin{equation}
    Y\sim  \dfrac{1}{2}\normalN(-2, 1)+\dfrac{1}{2}\normalN(2, 1)
\end{equation}
Define a pair of piecewise affine functions (see also Figure~\ref{fig:ex:invertible-ib-fail})
\begin{align}
f(x) = \left\{
\begin{aligned}
& x-4,\quad \text{for } x\geq 2,\\
& -x, \quad \text{for } -2\leq x< 2,\\
& x+4, \quad \text{for } -4\leq x< -2,\\
& (x+4)/5, \quad \text{for } x<-4.
\end{aligned}
\right.
\qquad
g(x) = \left\{
\begin{aligned}
& x-4,\quad \text{for } x\geq 4,\\
& -x+4, \quad \text{for } 2\leq x< 4,\\
& x, \quad \text{for } -2\leq x< 2,\\
& -x-4, \quad \text{for } -4\leq x<-2,\\
& (x+4)/5, \quad \text{for } x<-4.
\end{aligned}
\right.
\end{align}

\begin{figure}[h]
\begin{center}
\includegraphics[scale = 0.35]{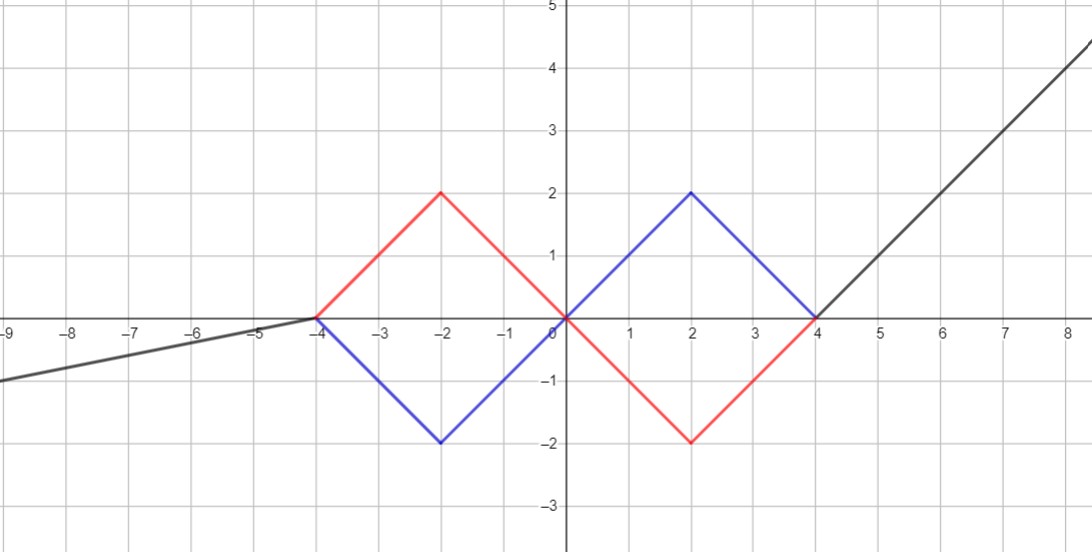}
\caption{Graphs of $f$ (black and red) and $g$ (black and blue) in Example~\ref{ex:invertible-ib-fail}.}\label{fig:ex:invertible-ib-fail}
\end{center}
\end{figure}

Then it is easy to see that $f(Y)$ and $g(Y)$ have the same distribution, but $f$ cannot be transformed into $g$ by an affine transformation.
\end{ex}

In order to prove Theorem~\ref{thm:dontneedu-strong} we need to show that for a mixture of Gaussians $P$ and a pair of piecewise affine functions $f, g$ if $f_{\sharp}P = g_{\sharp}P$, then $f = h\circ g$ for some invertible affine $h$. We first consider the case when $g$ is the identity.

\begin{lemma}\label{lem:affine-aut}
Let $Z~\sim \sum\limits_{j=1}^{J} \lambda_j\normalN(\mu_j, \Sigma_j)$. Assume that $f:\mathbb{R}^m \rightarrow \mathbb{R}^m$ is a continuous piecewise affine function such that $f(Z)\sim Z$. Then $f$ is affine.
\end{lemma}
\begin{proof} Since $Z$ has positive density at every point and $f(Z)\sim Z$ we must have $\dim f(\mathbb{R}^m) = m$.

If $f$ is not affine, then there exist an $(m-1)$-dimensional affine subspace $L$, $z_0\in L$ and $\delta>0$ such that the following holds: The subspace $L$ divides $B(z_0, \delta)$ into two sets (formally, these  are ``half-balls'') $B^+$ and $B^-$ such that $f_+(z):=f|_{B^+}(z) = A_1z+b_1$ and $f_-(z):=f|_{B^-}(z) = A_2z+b_2$, where $(A_1, b_1)\neq (A_2, b_2)$ and $A_1, A_2$ are invertible.

Since $f(Z)\sim Z$ we have
\[
\{f_+(\mu_1),  \ldots, f_+(\mu_J)\} = \{\mu_1, \mu_2, \ldots, \mu_K\} = \{f_-(\mu_1),  \ldots, f_-(\mu_J)\}
\]
as multisets (i.e. including repetitions). Let $\mu_* = \dfrac{1}{J}\sum\limits_{j=1}^J \mu_j$. Then, since $f_+$ and $f_-$ are affine we get $f_+(\mu_*) = f_-(\mu_*) = \mu_*$. By translating $Y$ and adjusting $f$ accordingly, we may assume that $\mu^* = 0$. In this case, $b_1 = b_2 = 0$.
Moreover, since $f_+(z) = f_{-}(z)$ for $z\in L$, we get
\begin{equation}\label{eq:hyperplane0}
    (A_1^{-1}A_2)(z) = z\quad  \text{for all } z\in L.
\end{equation}

Finally, since $f(Y)\sim Y$, we have
\[ \{A_1\Sigma_1 A_1^T, \ldots, A_1\Sigma_J A_1^T\} =\{\Sigma_1, \Sigma_2, \ldots \Sigma_J\} =\{A_2\Sigma_1 A_2^T, \ldots, A_2\Sigma_J A_2^T\},
\]
as multisets (i.e. including repetitions). This implies that
\[\prod\limits_{j=1}^J\det\left( A_1\Sigma_jA_1^T\right) = \prod\limits_{j=1}^J\det\left( \Sigma_j\right) = \prod\limits_{j=1}^J\det\left( A_2\Sigma_jA_2^T\right).\]
Hence, $\det(A_1)^2 = \det(A_2)^2 = 1$, and $\det(A_1^{-1}A_2)^2 = 1$. By \eqref{eq:hyperplane0}, $A_1^{-1}A_2$ is the identity map on $L$. Let $v$ be a unit vector orthogonal to $L$ (in the direction of $B^{+}$). Then we get that either $A_1^{-1}A_2 v = v$, or $A_1^{-1}A_2 v = -v$. In the latter case $A_1(y_0+(\delta/2) v) = A_2(y_0-(\delta/2) v)$, which means that $f$ is not injective. This contradicts Lemma~\ref{lem:preimage-size}. Therefore, we must have $A_1^{-1}A_2 v = v$, and so, by \eqref{eq:hyperplane0}, $A_1 = A_2$.

Therefore, $f_+ = f_-$, which contradicts $(A_1, b_1)\neq (A_2, b_2)$.
It follows that $f$ must be affine.
\end{proof}

\begin{theorem}\label{thm:identif-inv-affine-strong}
Let $f, g:\mathbb{R}^m \rightarrow \mathbb{R}^n$ be continuous invertible piecewise affine functions. Let $Z~\sim \sum\limits_{i=1}^{J} \lambda_i\normalN(\mu_i, \Sigma_i)$ and $Z' \sim \sum\limits_{j=1}^{J'} \lambda_j'\normalN(\mu_j', \Sigma_j')$ be a pair of GMMs (in reduced form). Suppose that $f(Z)$ and $g(Z')$ are equally distributed.

 Then there exists an affine transformation $h:\mathbb{R}^m \rightarrow \mathbb{R}^m$ such that $h(Z) \equiv Z'$ and $g = f\circ h^{-1}$.
\end{theorem}
\begin{proof}
By Theorem~\ref{thm:identif-inv-affine}, there exists an invertible affine transformation $h_0:\mathbb{R}^m \rightarrow \mathbb{R}^m$ such that $h_0(Z) = Z'$. Then, $f(Z) \sim  g(h_0(Z))$, and since $g$ and $h_0$ are invertible, we can rewrite this as $Z\sim (h_0^{-1}\circ g^{-1} \circ f )(Z) $. By Lemma~\ref{lem:affine-aut}, $(h_0^{-1}\circ g^{-1} \circ f )$ is affine, i.e. there exists an invertible affine map $h_1$ such that
\[ h_0^{-1}\circ g^{-1} \circ f  = h_1 \quad \Leftrightarrow \quad f = g\circ (h_0\circ h_1)\]
Hence the claim of the theorem holds for $h = h_0\circ h_1$.
\end{proof}

\begin{proof}[Proof of Theorem~\ref{thm:dontneedu-strong}]
Immediately follows from Theorems~\ref{thm:dontneedu} and \ref{thm:identif-inv-affine-strong}.
\end{proof}

\subsection{Identifiability under assumption~\ref{assm:invae}}\label{app:assm:invae}

In this section we discuss the case \ref{assm:invae}. In particular, show that in \eqref{eq:gmm-ivae} under the weaker assumption \ref{assm:invae}, $f$ is identifiable up to an affine transformation on the preimage of every connected open set onto which $f$ is injective.

\begin{theorem}
Let $f, g:\mathbb{R}^m \rightarrow \mathbb{R}^n$ be continuous piecewise affine functions satisfying~\textup{\ref{assm:invae}}. 

Let $Z~\sim \sum\limits_{i=1}^{J} \lambda_i\normalN(\mu_i, \Sigma_i)$ and $Z' \sim \sum\limits_{j=1}^{J'} \lambda_j'\normalN(\mu_j', \Sigma_j')$ be a pair of variables with GMM distribution (in reduced form). Suppose that $f(Z)$ and $g(Z')$ are equally distributed.

 Let $\mathcal{D}\subseteq \mathbb{R}^n$ be a connected open set such that $f$ and $g$ are injective onto $\mathcal{D}$. Then there exists an affine transformation $h:\mathbb{R}^m \rightarrow \mathbb{R}^m$ such that $h(Z) \equiv Z'$ and $g(z) = (f\circ h^{-1})(z)$ for every $z\in g^{-1}(\mathcal{D})$.
\end{theorem}
\begin{proof}
Similarly, as in the proof of Theorem~\ref{thm:identif-inv-affine-strong}, by Theorem~\ref{thm:identif-inv-affine}, there exists an invertible affine transformation $h_0:\mathbb{R}^m \rightarrow \mathbb{R}^m$ such that $h_0(Z) = Z'$. Then, $f(Z) \sim  g(h_0(Z))$, and since $g$ is invertible on $\mathcal{D}$ and $h_0$ is invertible, we can rewrite this as $Z\sim (h_0^{-1}\circ g^{-1} \circ f )(Z)$ on $f^{-1}(\mathcal{D})$. Since $f$ is invertible and continuous piecewise affine, $f^{-1}(\mathcal{D})$ is an open connnected set. Therefore, applying Lemma~\ref{lem:affine-aut} on $f^{-1}(\mathcal{D})$, we deduce that $(h_0^{-1}\circ g^{-1} \circ f )$ is affine on $f^{-1}(\mathcal{D})$, i.e. there exists an invertible affine map $h_1$ such that
\[ h_0^{-1}\circ g^{-1} \circ f  = h_1 \quad \Leftrightarrow \quad f = g\circ (h_0\circ h_1) \quad \text{on $f^{-1}(\mathcal{D})$}\]
Therefore, for $h = (h_0\circ h_1)$, we have $g(y) = (f\circ h^{-1})(z)$ for every $z\in g^{-1}(\mathcal{D})$.
\end{proof}

\begin{remark}
Let $f$ be a continuous piecewise affine function that satisfies~\ref{assm:invae}. Denote
\[S = \{x\in \mathbb{R}^n\, :\, |f^{-1}(\{x\})|>1\}\subseteq f(\mathbb{R}^m). \]
Recall that assumption~\ref{assm:invae} says that $S$ has measure zero in $f(\mathbb{R}^m)$.

We claim that~\ref{assm:invae} implies that for every $x\in S$ in fact $|f^{-1}(\{x\})| = \infty$. Indeed, if for all sufficiently small $\delta>0$ we have $\dim \left(B(x, \delta)\cap f(\mathbb{R}^m) \right)<m$, then $|f^{-1}(\{x\})| = \infty$ since $f$ is continuous piecewise affine. Otherwise,  using Corollary~\ref{obs:oa-generic}, we get that for every $\delta>0$ there exists a generic with respect to $f$ point $x_{\delta}\in B(x, \delta)\cap f(\mathbb{R}^m)$. Assumption~\ref{assm:invae} implies that $|f^{-1}(\{x_{\delta}\})| = 1$ for every $x_{\delta}$. Therefore, since $f$ is continuous piecewise affine we get that either $|f^{-1}(\{x\})| = 1$  or $|f^{-1}(\{x\})| = \infty$.
\end{remark}

\section{Identifiability of $Z$ up to a permutation, scaling and translation}\label{sec:icagmmproof}

Under \ref{assm:latent}, we have
\begin{equation}\label{eq:GMM-ICA}
    Z\sim  \sum\limits_{j = 1}^J \lambda_j \normalN (\mu_j, \Sigma_j),
\end{equation}
where $\Sigma_j$ is diagonal for every $j\in [J]$. In the setup of model~\eqref{eq:gmm-ivae} this just means that $Z_i\indep Z_j\mid U$.

Let $Y = AZ+b$, where $A:\mathbb{R}^m\rightarrow \mathbb{R}^m$ is an invertible linear map and $b\in \mathbb{R}^m$. Then $Y$ is also a GMM. We next show how $Z$ may be recovered from $Y$ up to a permutation, scaling, and translation.

\begin{theorem}\label{thm:gmm-ICA}
Let $J\geq 2$, and $\lambda_j>0$ for all $j\in [J]$. Let $Z = (Z_1, Z_2, \ldots, Z_m)$ be given by
\begin{equation}
    Z\sim \sum\limits_{j = 1}^{J} \lambda_{j} \normalN(\mu_{j}, \Sigma_{j})
\end{equation}
Assume that $\Sigma_j$ is diagonal for every $j\in [J]$. Let $Y = AZ+b$, where $A:\mathbb{R}^m\rightarrow \mathbb{R}^m$ is an invertible linear map and $b\in \mathbb{R}^m$. Moreover, assume that there exist indices $i_1, i_2\in [J]$, such that all numbers $(\left(\Sigma_{i_1}\right)_{tt}/\left(\Sigma_{i_2}\right)_{tt}\mid t\in [m])$ are distinct. Given $Y$, one can recover an invertible linear map $A': \mathbb{R}^m\rightarrow \mathbb{R}^m$, such that $(A')^{-1}A = QD$, where $Q$ is a permutation matrix and $D$ is a diagonal matrix with positive entries.
\end{theorem}
\begin{remark}
The translation $b$ is impossible to recover without stronger assumptions, as $b$ corresponds to an arbitrary translation in the $Z$ space. In other words, choice of $b$ determines the origin in the coordinate space of $Z$ and it can be completely arbitrary.
\end{remark}
\begin{remark}
A slightly different version of Theorem~\ref{thm:gmm-ICA} under different assumptions appeared in \cite{yang2021nonlinear}. The main difference is that \cite{yang2021nonlinear} assumed that $f$ is volume-preserving but nonlinear, whereas we restrict to the general (i.e. not necessarily volume-preserving) linear case.
\end{remark}

\begin{proof} Without loss of generality assume $i_1 = 1$ and $i_2 = 2$.

Let $\Sigma_i$ be the covariance matrices of $Z_i$ and let $\widetilde{\Sigma}_i$ be the covariance matrices of $Y_i$ for $i\in [J]$. Clearly
\begin{equation}
\widetilde{\Sigma}_i = A\Sigma_i A^T
\quad\text{for each $i\in[J]$}.
\end{equation}

The matrices $\widetilde{\Sigma}_i$ are PSD. Therefore, using SVD we can find PSD matrices $V_i$, such that for every $i\in [J]$,
\begin{equation}
    \widetilde{\Sigma}_i = V_iV_i^T. \end{equation}
Moreover, such a decomposition is unique up to an orthogonal matrix, i.e., for every pair of such decompositions $\widetilde{\Sigma}_i = V_iV_i^T = V_i'(V_i')^T$ there exists a unitary matrix $R$ such that $V_iR = V_i'$. Therefore, for every $i\in [J]$ there exists a matrix $R_i$, such that
\begin{equation}
    V_iR_i = A\Sigma_i^{1/2}
\end{equation}
In particular,
\begin{equation}
    V_1R_1\Sigma_1^{-1/2} = V_2R_2\Sigma_2^{-1/2} \quad \Rightarrow \quad R_1\left(\Sigma_1^{-1/2}\Sigma_2^{1/2}\right)R_2^{-1} = \left(V_1^{-1}V_2\right)
\end{equation}
Since $R_1$ and $R_2^{-1}$ are unitary and $\left(\Sigma_1^{-1/2}\Sigma_2^{1/2}\right)$ is diagonal, they can be determined from the SVD of $V_1^{-1}V_2$. Moreover, they can be determined uniquely up to a permutation matrix since all diagonal entries of $\Sigma_1^{-1/2}\Sigma_2^{1/2}$ are distinct. In other words, using SVD for $\left(V_1^{-1}V_2\right)$  we can find $R_1'$ such that for some permutation matrix $P$ we have
\begin{equation}
    V_1R_1'Q = A\Sigma_1^{1/2}, \quad \text{so, for }\quad A' := V_1R_1 \quad \text{we have}\quad (A')^{-1}A = Q\Sigma_1^{-1/2}.
\end{equation}
This concludes the proof.
\end{proof}

As an immediate corollary we can deduce the following theorem from Theorem~\ref{thm:dontneedu}.

\begin{theorem}\label{thm:ivae-up-to-perm}
Assume that $(U, Z, X)$ are distributed according to model $\eqref{eq:gmm-ivae}$ and that $f$ is weakly injective.
Suppose that $Z_i\indep Z_j \mid U$ for all $i\neq j$. Moreover, assume that there exist a pair of states $U = u_1$ and $U = u_2$ such that all $(\left(\Sigma_{u_1}\right)_{tt}/\left(\Sigma_{u_2}\right)_{tt}\mid t\in [m])$ are distinct.

Then $P(U, Z)$ is identifiable from $P(X)$ up to permutation, scaling ans translation of $Z_i$.
\end{theorem}
\begin{proof}
By Theorem~\ref{thm:dontneedu}, $\pr(Z)$ is identifiable from $\pr(X)$ up to an affine transformation. That is, we can reconstruct a random variable $Y$ from $\pr(X)$ which satisfies $Y = AZ+b$ for some invertible $A\in \mathbb{R}^{m\times m}$.

Now, by Theorem~\ref{thm:gmm-ICA}, we can find $A'$ such that $Z' = (A')^{-1}Y = QDZ+(A')^{-1}b$, where $Q$ is a permutation matrix and $D$ is a diagonal matrix. This means, that we can recover $Z$ up to permutation, shift and scaling of individual variables $Z_i$.
\end{proof}

\section{Identifiability of multivariate $U$ structure}\label{sec:latent-dag-setup}

When $k=1$, $P(Z)$ contains all the information about $P(U, Z)$, however, when $k>1$ (i.e. $U$ is multivariate), this may not be true anymore. It is not even obvious that $P(Z)$ must contain information about the true dimension of $U$. The distribution $P(U, Z)$ may contain interesting dependencies between individual variables $U_i$ and $Z_j$.

Previously, \cite{kivva2021learning} studied necessary and sufficient conditions for identifiability of $P(U)$ when $Z$ is observed under the so-called \emph{measurement model}. A key limitation of \citet{kivva2021learning} is that it requires the observed variables to be conditionally independent, which is not the case in our setting.
Ultimately, this is a consequence of $Z$ being unobserved: Previous work such as \citet{kivva2021learning} assumes there is only a single layer of hidden variables connected to the observations. In our setting, under \eqref{eq:gmm-ivae}, we need to recover $U$ from $Z$, the latter of which is unobserved. As a result, if we can only identify $Z$ up to an affine transformation (e.g., like in Theorem~\ref{thm:dontneedu});
i.e. we can only recover $Z' = AZ+b$, then it almost surely will not be conditionally factorial.  Hence, the results from \citet{kivva2021learning} cannot be applied directly for weak (e.g., up to affine transformation, or as in \citealp{khemakhem2020variational}) notions of identifiability of $Z$.

Luckily, in Section~\ref{sec:icagmmproof}, we showed how to recover the true $Z$ from $Z' = AZ+b$. This will enable us to identify $P(U)$ in Theorem~\ref{thm:main:highd}\ref{thm:main:highd:latentdag}. In the remainder of this appendix, we outline these details.

We say that a distribution $\pr(U,Z)$ satisfies the \emph{Markov property} with respect to the neighborhoods $\nbhd(Z_i)$ (cf. Definition~\ref{def:nbhd}) if
\begin{equation}
    \prob(U,Z) = \prob(U)\prod_{i} \prob(Z_i\mid \nbhd(Z_{i})).
\end{equation}

\begin{remark}
The neighborhoods $\nbhd(Z_i)$ define a bipartite graph between $(U_1,\ldots,U_k)$ and $(Z_1,\ldots,Z_m)$ that is described in \citet{kivva2021learning}. Since this graph is not needed for our purposes, we proceed without further mention of this graph. The assumptions below have been re-phrased accordingly.
\end{remark}

\cite{kivva2021learning} show that assumptions \ref{assm:twins}-\ref{assum:ssc} below are necessary for identifiability of $U$.

\begin{enumerate}[label=(L\arabic*)]
\item \label{assm:twins} (No twins)  For any $U_{i}\neq U_{j}$ we have $\nbhd(U_i)\neq\nbhd(U_j)$.
    \item \label{assm:maximal-old} (Maximality) There is no $U'$ such that:
    \begin{enumerate}
        \item $\pr(U', Z)$ is Markov with respect to the neighborhoods $\nbhd(Z_i)$ defined by $U'$;
        \item $U'$ is obtained from $U$ by splitting a hidden variable (equivalently, $U$ is obtained from $U'$ by merging a pair of vertices);
        \item $U'$ satisfies Assumption~\ref{assm:twins}.
    \end{enumerate}
    \item\label{assump:nondeg} (Nondegeneracy) The distribution over $(U, Z)$ satisfies:
    \begin{enumerate}[label=(\alph*)]
    \item\label{assump:strong:NZ} $\pr(U = u)>0$ for all $u$.
    \item\label{assump:strong:SDC} For all $Z'\subset Z$ and $u_1\ne u_2$, $\pr(Z'|\nbhd(Z') = u_1)\ne\pr(Z'|\nbhd(Z') = u_2)$, where $u_1$ and $u_2$ are distinct configurations of $\nbhd(Z')$.
    \end{enumerate}
    \item \label{assum:ssc} (Subset condition) For any pair of distinct variables $U_i, U_j$ the set $\nbhd(U_i)$ is not a subset of $\nbhd(U_j)$.
\end{enumerate}

We prove the following identifiability result.

\begin{theorem}\label{thm:ivae-latentdag}
Assume that $(U, Z, X)$ are distributed as in $\eqref{eq:gmm-ivae}$ and that $f$ satisfies~\ref{assm:invib}.
Assume further that~\ref{assm:latent}-\ref{assm:latentdag} hold and $P(U = u)>0$ for all $u$ in the domain of $U$.

Then $\dim(U)=k$, $\dim(U_j)$, $\pr(U, Z)$ are identifiable from $P(X)$ up to a permutation of variables $U_i$ and permutation, scaling and translation of variables $Z_i$. 
\end{theorem}
\begin{proof}
The assumptions of Theorem~\ref{thm:ivae-latentdag} are stronger than those of Theorem~\ref{thm:ivae-up-to-perm}, so by Theorem~\ref{thm:ivae-up-to-perm}, $P(Z)$ is identifiable up to a permutation, scaling and translation of $Z$. 

Combined with the positivity assumption $P(U = u)>0$, the assumptions \ref{assm:twins}-\ref{assum:ssc} are weaker than assumption \ref{assm:latentdag}. Indeed, \ref{assm:latentdag} (a) is equivalent to \ref{assump:nondeg} (b);  \ref{assm:latentdag} (c) is equivalent to  \ref{assum:ssc} and implies \ref{assm:twins}; and, finally, \ref{assm:latentdag} (b) and (c) together imply \ref{assm:maximal-old}.

Since $Z$ is identifiable up to a permutation, scaling and translation, $Z_i\indep Z_j \mid U$, and  assumptions \ref{assm:twins}-\ref{assum:ssc} hold, using \cite[Thm 3.2]{kivva2021learning}, we deduce that $\dim(U)=k$, $\dim(U_j)$, $\pr(U)$, and $\nbhd(U_i)$ are identifiable up to a permutation of the variables $U_i$. Finally, by the Markov Property, $\pr(U)$, $\nbhd(U_i)$ for all $i$, and the fact that $P(Z)$ is a finite GMM (that is identifiable)  are sufficient to recover $\pr(U, Z)$.
\end{proof}

\begin{remark}\label{rem:maxim-hierarchy}
As the proof indicates, assumptions \ref{assm:twins}-\ref{assum:ssc} are weaker than \ref{assm:latentdag}, so Theorem~\ref{thm:ivae-latentdag} implies part~\ref{thm:main:highd:latentdag} of Theorem~\ref{thm:main:highd}. 
\end{remark}

\section{Equivalence in iVAE}
\label{sec:ivae:equiv}

In this section we compare the equivalence relation up to which iVAE \citep{khemakhem2020variational} guarantees identifiability and equivalence up to an affine transformation. While iVAE achieves the best possible identifiability under the assumptions they make, we show that identifiability up to an affine transformation is considerably stronger.

\subsection{iVAE equivalence relation}

Recall that iVAE \citep{khemakhem2020variational} considers the following model, which differs from \eqref{eq:gmm-ivae} by assuming that $Z$ has conditionally factorial exponential family distribution:
\begin{align}
\label{eq:ivae-app}
\left.
\begin{aligned}
U=u
&\sim p(u) \\
[Z\given U=u]
&\sim \prod\limits_{i=1}^{m}\dfrac{Q_i(z_i)}{C(u)}\exp\left(\sum\limits_{j=1}^{t} T_{i, j}(z_i)\lambda_{i, j}(u)\right)\\
[X\given Z=z]
&\sim f(z)+\varepsilon, \quad \varepsilon\sim  \normalN(v,\sigma^{2})
\end{aligned}
\right\}
\implies
U\to Z\to X.
\end{align}
Here $T_i = (T_1, T_2, \ldots T_t)$ are sufficient statistics, $Q_i$ is the base measure and $\lambda_{i, j}$ parameters depending on $u$. iVAE defines the following equivalence relation:

\begin{defn}
    \begin{equation}\label{eq:ivae-equiv} (f, T, \sigma)\sim (f', T', \sigma')\quad  \Leftrightarrow\quad  \exists A, c:\quad T(f^{-1}(\{x\})) = A(T'((f')^{-1}(x)+c,
    \end{equation}
    where $A:\mathbb{R}^{mt}\rightarrow \mathbb{R}^{mt}$ is an invertible linear map, and $c\in \mathbb{R}^N$.
    \end{defn}

This type of identifiability allows for essentially any (synchronized) changes to $Z$ and $f$:
\begin{lemma}
\label{obs:ivae:equiv}
Let $\phi:\mathbb{R}^m\rightarrow \mathbb{R}^m$ be any invertible map. Let $f' = f\circ \phi$, and $T' = T\circ \phi$. Then $(f,T,\sigma)\sim(f',T',\sigma)$.

Moreover, if $Z$ has exponential family distribution with statistics $T$, then $Z' = \phi^{-1}(Z)$, has an exponential family distribution with statistics $T'$, and $f(Z)\sim f'(Z')$.
\end{lemma}
\begin{proof}
We have $(f')^{-1} = \phi^{-1}\circ f^{-1}$, so $T'\circ (f')^{-1} = T\circ f$. Hence $(f, T, \sigma)\sim (f', T', \sigma')$, where in \eqref{eq:ivae-equiv} $A$ is the identity map and $c = 0$.

Since $Z$ comes from an exponential family distribution, we can write
\begin{equation}
    \pr(Z\mid U) = h(Z)g(U)\exp(\lambda(U)T(Z)).
\end{equation}
Let $Z' = \phi^{-1}(Z)$. Then by the change of variable formula
\begin{equation}
    \pr(Z'\mid U) = \left(h(\phi(Z))\det|Jac(\phi(\bullet))|_{\bullet = \phi^{-1}(Z)}\right)g(U)\exp(\lambda(U)T(\phi(Z))),
\end{equation}
where $Jac(\phi)$ is the Jacobian of $\phi$. Hence $Z'$ indeed has an exponential family distribution with statistics $T'$. Clearly, $f'(Z') = (f\circ \phi\circ \phi^{-1})( Z) \equiv f(Z)$.
\end{proof}

\begin{remark}
\label{rem:ivae:basis}
In other words, the equivalence relation~\eqref{eq:ivae-equiv} allows an \emph{arbitrary} (possibly highly nonlinear) change of basis in the latent $Z$ space. In principle, this may indicate, that any meaningful analysis of the $Z$ space in this setup may be challenging.
\end{remark}

\begin{remark}
As in \citet{khemakhem2020variational}, the additional assumption that $Z$ has a conditionally factorial distribution imposes additional restrictions on $\phi$. In this case, $\phi: \mathbb{R}^m\rightarrow \mathbb{R}^m$ can be any invertible coordinatewise function $\phi(Z') = (\phi_1(z'_{1}), \phi_2(z'_{2}), \ldots \phi_2(z'_{m}))$.
\end{remark}

\subsection{GMMs give more robust identifiability}

The next result was also observed in~\cite{sorrenson2019disentanglement}. We present a slightly simplified proof for completeness.

If $\pr(Z|U)$ is a multivariate Gaussian distribution, then the sufficient statistics are given by
\begin{equation}\label{eq:normal-distr-statistics}
    T_m = (z_1, \ldots, z_m, z_1z_1, z_1z_2,  \ldots z_{m}z_m).
\end{equation}

\begin{remark}
For product measures, there are no cross-terms $z_iz_j$.
\end{remark}

\begin{prop}[{\citealp[Appendix B]{sorrenson2019disentanglement}}]
Assume that $(T_m, f, \sigma)\sim (T_m, f', \sigma')$, where $T_m$ is defined by \eqref{eq:normal-distr-statistics}. Then there exists an invertible linear map $M:\mathbb{R}^{m}\rightarrow \mathbb{R}^{m}$ and a vector $c\in \mathbb{R}^m$ such that $f^{-1}(\{x\}) = M(f')^{-1}(x)+c$ for every $x$.
\end{prop}
\begin{proof}
Let $z = f^{-1}(\{x\})$ and $z' = (f')^{-1}(x)$. By an assumption of the proposition there exists an invertible matrix $A:\mathbb{R}^{m+m^2}\rightarrow \mathbb{R}^{m+m^2}$ such that
\begin{equation}\label{eq:bigmatrix}
 \left(\begin{matrix}
 z_1\\
 z_2\\
 \vdots\\
 z_n\\
 z_1z_1\\
 z_1z_2\\
 \vdots\\
 z_mz_m
 \end{matrix}\right) =
 A\left(\begin{matrix}
 z'_1\\
 z'_2\\
 \vdots\\
 z'_n\\
 z'_1z'_1\\
 z'_1z'_2\\
 \vdots\\
 z'_mz'_m
 \end{matrix}\right)+b
\end{equation}
This means that for every $i$ there exists a polynomial $p_i$ of degree at most $2$ such that $z_i = p_i(z'_1, \ldots, z_m')$. Assume that for some $i$, we have $\deg(p_i) = 2$. Then it is easy to verify (say, by using lexicographical order on monomials) that $\deg(p_i^2) = 4$. If $z'$ is defined on an open neighbourhood, we get a contradiction with \eqref{eq:bigmatrix} as $z_i^2$ can be written as a degree-2 polynomial over variables $z_j'$. Therefore, every $p_i$ is a polynomial of degree at most 1. But this means that that $z = Mz'+c$ for some matrix $M$ and a vector $c$. Moreover, since $A$ is invertible, $M$ is invertible as well.
\end{proof}

\section{Conditions on ReLU Neural Network that guarantee that it is an observable injection}\label{sec:relu-inv}

For completeness, in this section we provide simple sufficient conditions on ReLU architectures that guarantee that it is an observable injection (cf. \ref{assm:invae}) and simple sufficient conditions on leaky-ReLU architectures which guarantee that it is injection (cf. \ref{assm:inv}). For a more comprehensive account of identifiability in ReLU networks, see \citet{stock2021embedding}.

We recall the definitions of ReLU and leaky-ReLU (with parameter $a>0,\ a\neq 1$) activation functions
\begin{align}
\RELU(x) = \left\{
\begin{aligned}
& x,\quad \text{for } x> 0,\\
& 0, \quad \text{for } x\leq 0,\\
\end{aligned}
\right.
\qquad
\LRELU(x) = \left\{
\begin{aligned}
& x,\quad \text{for } x>0,\\
& a\cdot x, \quad \text{for } x\leq 0.\\
\end{aligned}
\right.
\end{align}
A standard choice of $a$ for leaky-ReLU is $a = 0.01$.

\begin{definition}
Let $\Aff(n_1, n_2)$ denote the set of affine maps $h:\mathbb{R}^{n_1}\rightarrow \mathbb{R}^{n_2}$.
\end{definition}

 Let $\sigma:\mathbb{R}\rightarrow \mathbb{R}$ be a general activation function. For a vector $x\in \mathbb{R}^t$, $\sigma(x)$ is a vector obtained from $x$ by applying $\sigma$ coordinatewise.

\begin{definition}
Let $n_1, n_2, \ldots, n_t\geq n_0 = m$ and $\sigma$ be an activation function. Define
\begin{equation}
    \mathcal{F}^{n_0, \ldots, n_t}_{\sigma} = \{h_t\circ \sigma \circ h_{t-1} \circ \sigma \circ \ldots \sigma\circ h_1 \mid h_i \in \Aff(n_{i-1}, n_i)\}
\end{equation}
\begin{equation}
    \mathcal{F}^{m\hookrightarrow n}_{\sigma} = \bigcup_{t = 1}^{\infty}\bigcup_{\ n_1, n_2, \ldots, n_t\geq n_0,\   n_0 = m,\  n_t = n} \mathcal{F}^{n_0, \ldots, n_t}_{\sigma}
\end{equation}
\end{definition}

\begin{remark}
The function families $\mathcal{F}_{\RELU}^{m\hookrightarrow n}$, $\mathcal{F}_{\LRELU}^{m\hookrightarrow n}$ are genuinely nonparametric: There is no bound on the number of layers. 
\end{remark}

\begin{remark}
 In the arguments below we do not rely on the fact that the activation function is the same on every layer, or even the same across the nodes of the same layer. However, we will give proofs only in this case, to simplify the presentation.
\end{remark}

\begin{remark}\label{rem:relu-prior-work}
ReLU networks under similar assumptions were also studied in \cite{khemakhem2020ice}.
\end{remark}

\begin{lemma} \label{obs:relu-inv}
Let $f = h_t\circ \sigma \circ h_{t-1} \circ \sigma \circ \ldots \sigma\circ h_1\in \mathcal{F}^{m\hookrightarrow n}_{\RELU}$. Assume that $m = n_0\leq n_1\leq \ldots \leq n_t = n$, 
and $\dim(f(\mathbb{R}^m)) = m$. Then for almost all $y\in f(\mathbb{R}^m)$ there exists $\delta_y$ such that $f^{-1}$ is a well-defined affine function on $B(y, \delta_y)\cap f(\mathbb{R}^m)$.
\end{lemma}
\begin{proof} We prove the claim by induction on the depth of the NN. If $t =1$, we have $f = h_1$ and the claim is trivial. Assume that we already proved the lemma for all $t\leq s-1$. We prove the claim for $t= s$. We can write $f$ as $f = h_{t}\circ \sigma \circ g $ where $g\in \mathcal{F}^{m\hookrightarrow n_{t-1}}_{\RELU}$.

 Since $\dim(f(\mathbb{R}^m)) = m$, the map $h_t$ has full column rank. Additionally, denoting by $\mathcal{D} = \{x\in \mathbb{R}^{n_{t-1}} \mid x_i>0,\ \forall i\in [n_{t-1}]\}$ the domain on which $\sigma$ is injective, we get $g(\mathbb{R}^m)\cap \mathcal{D}$ has positive measure in $g(\mathbb{R}^m)$. Moreover, by the induction assumption, $g$ satisfies conclusion of the lemma, i.e., there exists a set $S$ of measure 0 in $g(\mathbb{R}^m)$ such that for any $y\in g(\mathbb{R}^m)\setminus S$ there exists a $\delta_y>0$ such that $g^{-1}$ is a well-defined affine function on $B(y, \delta_{y})\cap g(\mathbb{R}^m)$.
 Since $h_t$ has full column rank,  $f^{-1}$ is a well-defined affine function on $B(x, \delta_x)\cap f(\mathbb{R}^m)$ for every $x = (f\circ \sigma)(y)$ where $y\in \left(g(\mathbb{R}^m)\setminus S\right)\cap \mathcal{D}$. 
 Clearly, such $x$ form a set of full measure in $f(\mathbb{R}^m)$.
\end{proof}

\begin{corollary}
Let $f = h_t\circ \sigma \circ h_{t-1} \circ \sigma \circ \ldots \sigma\circ h_1\in \mathcal{F}^{m\hookrightarrow n}_{\RELU}$. Assume that $m = n_0\leq n_1\leq \ldots \leq n_t = n$, and $\dim(f(\mathbb{R}^m)) = m$, then $f$ satisfies~\textup{\ref{assm:invae}}. 
\end{corollary}
\begin{proof}
Immediately follows from Lemma~\ref{obs:relu-inv}.
\end{proof}

\begin{lemma} \label{obs:leakyrelu-inv}
Let $f = h_t\circ \sigma \circ h_{t-1} \circ \sigma \circ \ldots \sigma\circ h_1\in \mathcal{F}^{m\hookrightarrow n}_{\LRELU}$. Assume that $m = n_0\leq n_1\leq \ldots \leq n_k = n$ and every $h_i$ is invertible. Then for almost all $y\in f(\mathbb{R}^m)$ there exists $\delta_y$ such that $f^{-1}$ is a well-defined affine function on $B(y, \delta_y)\cap f(\mathbb{R}^m)$.
\end{lemma}
\begin{proof}
Clearly, any $f = h_t\circ \sigma \circ h_{t-1} \circ \sigma \circ \ldots \sigma\circ h_1\in \mathcal{F}_{\LRELU}$ is a piecewise affine function. The $\LRELU$ activation function is invertible, so $f$ is invertible. Finally, since $f$ is a piecewise affine transformation, for almost all $y\in B(x_0, \delta)$ there exists $\delta_y$ such that $f^{-1}$ is an affine function on $B(y, \delta_y)$.
\end{proof}
\begin{cor}
Let $f = h_t\circ \sigma \circ h_{t-1} \circ \sigma \circ \ldots \sigma\circ h_1\in \mathcal{F}^{m\hookrightarrow n}_{\LRELU}$. Assume that $m = n_0\leq n_1\leq \ldots \leq n_t = n$, then generically $f$ satisfies~\textup{\ref{assm:inv}}. 
\end{cor}
\begin{proof}
Generically, every $h_i$ has full column rank, and so is injective. Since $\LRELU$ is injective, we get that $f$ is injective.
\end{proof}

We conclude with an example of a very simple $\LRELU$ NN that is not even weakly injective.
\begin{ex}\label{ex:noninv}
Let $\sigma(x) = x$ for $x\geq 0$ and $\sigma(x) = x/2$ for $x<0$. Let $h_1:\mathbb{R}\rightarrow \mathbb{R}^{2}$ defined as $h_1(x) = (x, -x)$. Then $\sigma\circ h_1(x) = (x, -x/2)$ if $x\geq 0$ and $\sigma\circ h_1(x) = (x/2, -x)$ if $x<0$. Let $h_2:\mathbb{R}^2\rightarrow \mathbb{R}^2$ given by
\[ h_2 = \left(\begin{matrix}
1 & -1\\
1 & 1
\end{matrix}\right)
\]
Then $(h_2\circ \sigma\circ h_1)(x) = (3x/2, x/2)$ for $x\geq 0$ and $(h_2\circ \sigma\circ h_1)(x) = (3x/2, -x/2)$ for $x<0$ (see Figure~\ref{fig:ex:noninv}).
Let $h_3(x, y) = y$. Then $f(x) : = (h_3\circ \sigma \circ h_2\circ \sigma\circ h_1)(x) = |x|/2$. By Remark~\ref{rem:non-invert-f-example}, this implies that $f$ is not invertible at every point except 0.
\begin{figure}[h]
\begin{center}
\includegraphics[scale = 0.4]{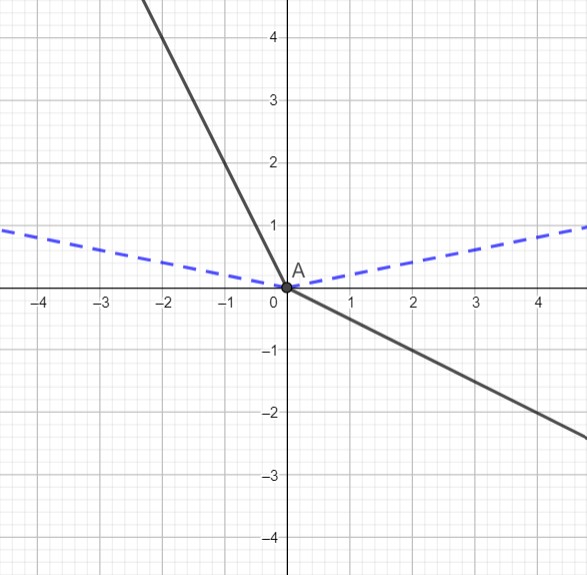}
\caption{Graphs of $\sigma\circ h_1$ (black) and $h_2\circ \sigma \circ h_1$ (blue) in Example~\ref{ex:noninv} }\label{fig:ex:noninv}.
\end{center}
\end{figure}
\end{ex}

\addtocounter{section}{+1}

\section{Experiment details}\label{app:expt}

\subsection{Metrics}\label{sec:metrics}

Previous work has relied on the Mean Correlation Coefficient (MCC) as a metric to quantify identifiability. For consistency with previous work, we report this metric, but also propose a new metric to quantify identifiability up to an affine transformation.
There are two challenges in designing such a metric:
Firstly, for two Gaussian mixtures, standard distance metrices such as TV-distance or KL-divergence do not have a closed form. Secondly, we need to find an affine map $A$ that best aligns a pair of Gaussian mixtures. 
Therefore, developing a metric to quantify identifiability up to an affine transformation has natural challenges. 
We propose $\dist_{\Aff, L2}$, defined below, as an additional metric in this setting.

\paragraph{Measuring loss}
In this work, we consider two different metrics.
For a pair of distributions $p_1, p_2$, we define $\dist_{\Aff, L2}$ loss as
\begin{equation}
  \dist_{\Aff, L2}(p_1, p_2) =  \min\limits_{\substack{A:\mathbb{R}^m\rightarrow \mathbb{R}^m,\\ \text{affine}}} 
  \Delta_{L_2}(A_{\sharp}p_1, p_2), \quad \text{where}\quad   \Delta_{L_2}(p_1, p_2) = \dfrac{ \norm{p_1 - p_2}_{L_2}}{\norm{p_1}^{1/2}_{L_2}\norm{p_2}_{L_2}^{1/2}}
\end{equation}
The other metric we consider is the Mean Correlation Coefficient (MCC) metric which had been used in prior works \citep{khemakhem2020ice, willetts2021don}. See \citet[Appendix A.2]{khemakhem2020ice} for a detailed discussion. There are two versions of MCC that have been used:
\begin{itemize}
    \item The \emph{strong} MCC is defined to be the MCC before alignment via the affine map $A$.
    \item The \emph{weak} MCC is defined to be the MCC after alignment.
\end{itemize}
In our experiments, we report both the strong MCC and weak MCC. Moreover, all reported MCCs are out-of-sample, i.e. the optimal affine map $A$ is computed over half the dataset and then reused for the other half of the dataset.

\paragraph{Alignment}
To find the affine map $A$ that best aligns the two GMMs, we use two approaches.
One approach is to use Canonical Correlation Analysis (CCA) as was done in prior works in computing MCC.

We describe an alternative approach now.
Given two GMMs, we iterate over all permutations of the components and for each fixed permutation, we find the best map $A$ that maps the components accordingly. In an ideal setting, we would want to find $A$ to align not just the means but also the covariance matrices but unfortunately this is a challenging optimization problem. Therefore, we instead find $A$ that maps the means of the first GMM to the means of the second GMM. The map $A$ can be found by solving a least-squares optimization problem which is straightforward using a Singular Value Decomposition (SVD). In practice, we find that this technique of matching the means works well.

\subsection{Implementation}

For VaDE \citep{jiang2016variational}, we use the implementation available at \url{https://github.com/mperezcarrasco/Pytorch-VaDE}.
For MFCVAE \citep{falck2021multi}, we use the author implementation available at \url{https://github.com/FabianFalck/mfcvae}. For iVAE \citep{khemakhem2020variational}, we use the implementation available at \url{https://github.com/MatthewWilletts/algostability}. Experiments were performed on an NVIDIA Tesla K80 GPU with 12GB memory.

\subsection{Setup}

Our experiments consist of three different setups, designed to probe different aspects of identifiability. First, we checked the exact log-likelihood for a unique global minimizer on simple toy models (Appendix~\ref{app:expt:mle}). We then used VaDE \citep{jiang2016variational} to train a practical VAE on a simulated dataset where the ground truth latent space is known (Appendix~\ref{app:expt:sim}). Finally, we compared the performance of MFCVAE \citep{falck2021multi} against iVAE on MNIST (Appendix~\ref{app:expt:real}). The last experiment is based on previous work by \cite{willetts2021don} that compares iVAE to VaDE; we successfully replicated these experiments using MFCVAE as an additional baseline that closely aligns with our assumptions.

The fact that our theory closely aligns with and replicates existing empirical work illustrates that the model \eqref{eq:gmm-ivae} is not merely a theoretical curiosity, but in fact practically relevant in modern applications. In our view, this is a significant advantage compared to related work.

\subsubsection{Maximum likelihood} 
\label{app:expt:mle}

We simulated random models of the form \eqref{eq:defn:nica} as follows:
\begin{enumerate}
\item Fix $J=2$ or $J=3$;
\item Randomly select $(\lambda_{1},\ldots,\lambda_{J})$ from a uniform grid by discretizing the simplex;
\item Randomly select $(\mu_{1},\ldots,\mu_{J})$ from a uniform grid on the hypercube;
\item Randomly select coefficients $(\alpha_{1},\alpha_{2})$, weights $(\beta_{1},\beta_{2})$, and biases $(\pi_{1},\pi_{2})$ from a uniform grid on the hypercube.
\end{enumerate}
Given these parameters, the prior $P(Z)$ is defined as in \eqref{eq:assm:mixture} and the decoder $f$ is defined to be the following single-layer ReLU network
\begin{align*}
f(z)
= \alpha_{1}\RELU(\beta_{1}z+\pi_{1}) + \alpha_{2}\RELU(\beta_{2}z+\pi_{2}).
\end{align*}
As a result of the simulation mechanism, the following important cases of misspecification naturally arise:
\begin{itemize}
\item We allow $\lambda_{j}=0$, i.e. the model allows for $J=3$ components, but the true model only has two nontrivial components.
\item We allow $\alpha_{j}=0$ and $\beta_{j}=0$, i.e. the model allows for up to two neurons in the hidden layer, but the true model only has one nontrivial neuron.
\item $f$ is not forced to be injective or even weakly injective, i.e. assumptions \ref{assm:invib}-\ref{assm:inv} are not checked explicitly. 
\end{itemize}

After generating a pair $(f,P(Z))$, the exact negative log-likelihood is approximated via numerical integration. An exhaustive grid search is performed over all parameters to identify the global minimizers. The computational cost of this step limited the complexity of the models that could be tested, hence the restriction to simple toy models in this experiment. In all runs, the ground truth was the unique global minimizer of the negative log-likelihood, as predicted by our theory. Since the problem is nonconvex, there often exist additional (non-global) local minima (see e.g. Figure~\ref{fig:mle}), however, the global minimizer is always unique up to affine equivalence. That is, due to affine equivalence, in some cases there is more than one global minimizer, but in all such cases it is easy to check that the different minimizers are indeed affinely equivalent. Multiple minimizers also arise when certain parameters (e.g. $\lambda_{j}$ or $\alpha_{j}$) vanish, again, these are easily checked.

\subsubsection{Simulated data}
\label{app:expt:sim}

We consider 4 synthetic datasets described below: Pinwheel and three different copies of the ``Random parallelograms" dataset

See Section~\ref{sec:expt} for results of the simulated experiments on the ``pinwheels'' dataset (see \citealp{johnson2016composing}). In those experiments we use 5000 samples and set $m=n=2$. In that experiment we used the same neural network architecture as discussed below for ``Random parallelograms".

We simulate an artificial dataset ``Random parallelograms'' as follows: We generate 3 randomly oriented parallelograms in the plane. After that, an $n$-dimensional observed distribution is obtained by sampling points uniformly at random from these parallelograms and by adding Gaussian noise to every sampled point.

We fit VaDE to each (observed) dataset 5 times (see Figures~\ref{fig:pinwheel3}, \ref{fig:rectangles1}-\ref{fig:rectangles3}). Let $Z^{(1)}, Z^{(2)}, \ldots, Z^{(5)}$ be the learned latent spaces. 
For every pair $Z^{(i)}, Z^{(j)}$ we evaluate the MCC and  $\dist_{\Aff, L2}$ loss. We report means of the MCCs/losses and their standard deviations in Table~\ref{tab: rectangles}.

For the VaDE training, we use a sequential neural network architecture with LeakyReLU activations for the encoder, with four fully connected layers of the following dimentions: $n\rightarrow 64\rightarrow 512\rightarrow 64\rightarrow m$.
For the decoder, we use a sequential neural network architecture with LeakyReLU activations, with four fully connected layers of the following dimentions: $m\rightarrow 64 \rightarrow 512 \rightarrow 512 \rightarrow n$.
We pretrain the autoencoder for 15 epochs and then run VaDE training for 20 epochs.

In all experiments with simulated data we set $m = 2$. We set the number of observed samples to be $5000$.

\begin{table}[htbp!]
\centering
\begin{tabular}{|c c c c |}
 \hline
 Dataset & $\dist_{\Aff, L2}$ & Strong MCC & Weak MCC \\
 \hline
 Random parallelograms \#1 & 0.1542 (0.150) & 0.86 (0.09) & 0.99 (0.003)\\
 Random parallelograms \#2 & 0.1231 (0.076) & 0.83 (0.12) & 0.99 (0.003)\\
 Random parallelograms \#3 & 0.578 (0.301) & 0.91 (0.08) & 0.99 (0.001)\\
 [1ex]
 \hline
\end{tabular}
\caption{Mean (std)  $\dist_{\Aff, L2}$ distance (lower is better) and Mean (std) MCC (higher is better) for synthetic data}
\label{tab: rectangles}
\end{table}

\begin{figure}[h!]
\includegraphics[width = \textwidth]{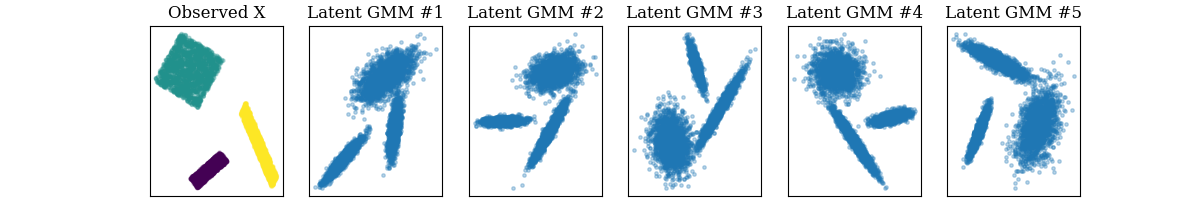}
\caption{Recovered latent spaces for 5 runs of VaDE on ``Random parallelograms'' dataset \#1 with 3 clusters}\label{fig:rectangles1}
\end{figure}

\begin{figure}[h!]
\includegraphics[width = \textwidth]{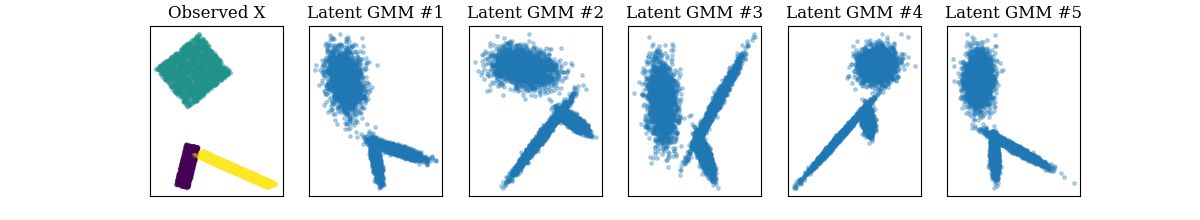}
\caption{Recovered latent spaces for 5 runs of VaDE on ``Random parallelograms'' dataset \#2 with 3 clusters}\label{fig:rectangles2}
\end{figure}

\begin{figure}[h!]
\includegraphics[width = \textwidth]{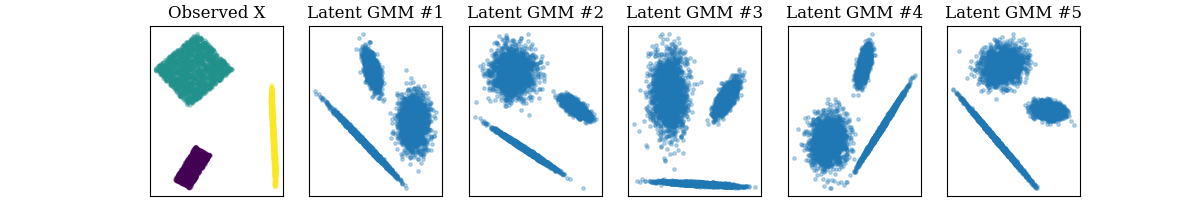}
\caption{Recovered latent spaces for 5 runs of VaDE on ``Random parallelograms'' dataset \#3 with 3 clusters}\label{fig:rectangles3}
\end{figure}

\subsubsection{Real data}
\label{app:expt:real} 

We run MFCVAE \citep{falck2021multi} on the MNIST dataset 10 times with different initializations. For all the 45 pairs of runs, we compute the strong MCC (before alignment) and weak MCC (after alignment with CCA of dimension $5$).
For these experiments, we omit the $\dist_{\Aff, L2}$ metric since it's computationally infeasible with a large number of components.
The mean and standard deviation of the MCCs are reported in Table~\ref{tab: mnist_mcc}. 
As a baseline, we also report the same metrics for 10 runs of iVAE \citep{khemakhem2020variational} on identical architecture and latent dimension, but recall that iVAE has additional access to the true digit labels $U$.

\begin{table}[htbp!]
\centering
\begin{tabular}{|c c c | c c|}
 \hline
 Architecture & Model & Activation & Strong MCC & Weak MCC \\
 \hline
 \multirow{3}{*}{Arch1} & MFCVAE & ReLU &
 0.7 (0.07) & 0.91 (0.05)\\
 & MFCVAE & LeakyReLU & 0.69 (0.06) & 0.94 (0.02)\\
 & iVAE & LeakyReLU & 0.65 (0.07) & 0.88 (0.07)\\
 \hline
 \multirow{3}{*}{Arch2} &
 MFCVAE & ReLU & 0.69 (0.07) & 0.89 (0.08)\\
 & MFCVAE & LeakyReLU & 0.69 (0.06) & 0.92 (0.03)\\
 & iVAE & LeakyReLU & 0.64 (0.07) & 0.87 (0.04)\\
 \hline
 \multirow{3}{*}{Arch3} &
 MFCVAE & ReLU & 0.69 (0.07) & 0.86 (0.08)\\
 & MFCVAE & LeakyReLU & 0.70 (0.05) & 0.92 (0.03)\\
 & iVAE & LeakyReLU & 0.67 (0.06) & 0.87 (0.05)\\
 [1ex]
 \hline
\end{tabular}
\caption{Mean and standard deviation of the MCCs (higher is better) across various models, architectures and activations}
\label{tab: mnist_mcc}
\end{table}

As recommended in \cite{falck2021multi}, we set the dimension of the latent space to be $5$ and number of components to be $25$. No hyperparameter tuning was done. The architectures we use are as follows:
\begin{itemize}
    \item Arch1: The encoder is a sequential neural network architecture with fully connected layers of dimensions $n\rightarrow 500\rightarrow 1000\rightarrow m$. The decoder is also a sequential neural network architecture with fully connected layers of dimensions $m \rightarrow 500 \rightarrow 500 \rightarrow n$.
    \item Arch2: The encoder is a sequential neural network architecture that is fully connected with dimensions $n\rightarrow 256\rightarrow 512\rightarrow 512\rightarrow m$. The decoder is similarly a sequential neural network architecture with fully connected layers of dimensions $m \rightarrow 512 \rightarrow 256 \rightarrow n$.
    \item Arch3: The encoder is a sequential neural network architecture that is fully connected with dimensions $n\rightarrow 128\rightarrow 256\rightarrow 128\rightarrow 128\rightarrow m$. The decoder is again a sequential neural network architecture with fully connected layers of dimensions $m \rightarrow 128 \rightarrow 128 \rightarrow n$.
\end{itemize}

The work \cite{willetts2021don} ran extensive experiments comparing VaDE and iVAE. We augment these experiments by using MFCVAE instead of VaDE. 
We observe that even without access to $U$, MFCVAE has competitive performance (stability) in recovering the latent space as compared to iVAE which has full access to $U$. This offers strong evidence for stability of training, as predicted by our theory.

For purely illustrative purposes, we also show the output of MFCVAE on MNIST. In Figure~\ref{fig: gen}, we show samples synthetically generated from each learnt cluster. In Figure~\ref{fig: recons}, we visualize the true datapoint $x$ and the corresponding reconstructed $\widehat{x}$  for four different datapoints in each cluster.
For similar experiments on other datasets and other architectures, we refer the reader to \cite{falck2021multi}.

\begin{figure}[htbp!]
    \centering
    \begin{subfigure}{.45\textwidth}
      \centering
      \includegraphics[width = \textwidth]{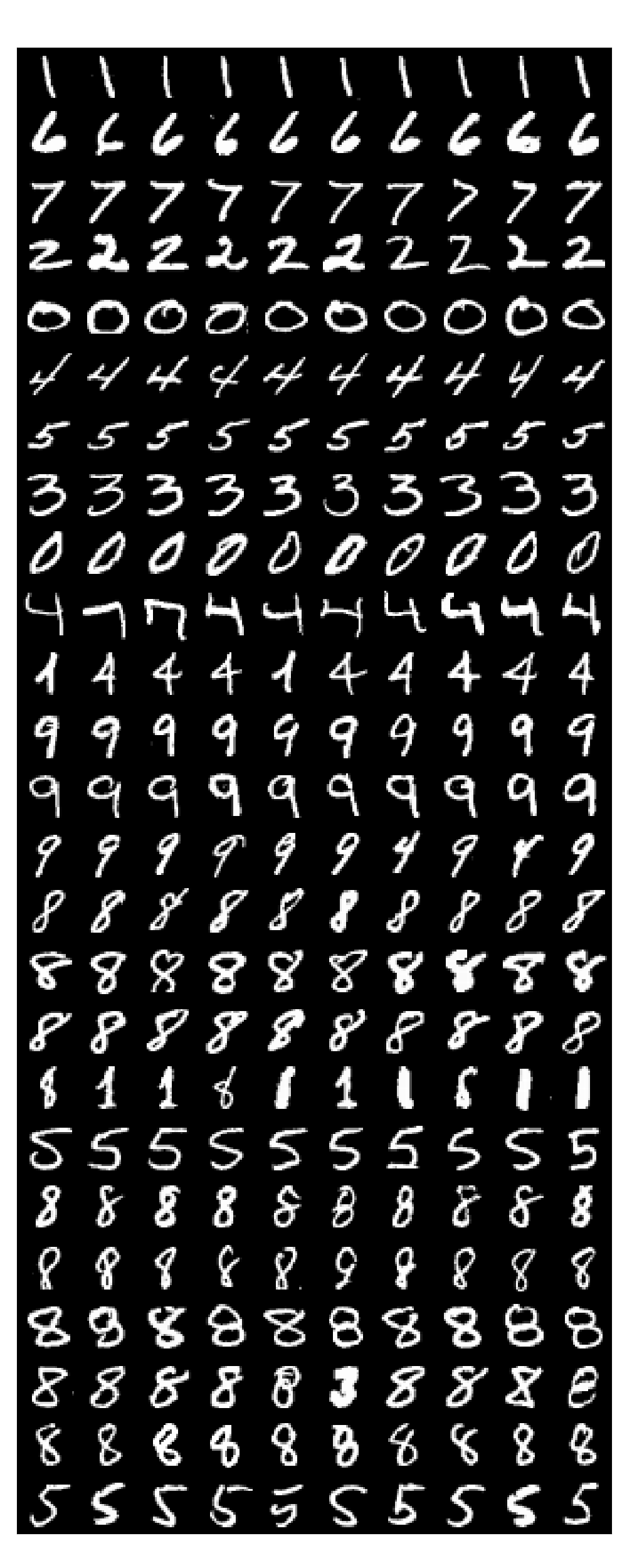}
        \caption{Arch1}\label{fig: gen1}
    \end{subfigure}
    \begin{subfigure}{.45\textwidth}
        \centering
        \includegraphics[width = \textwidth]{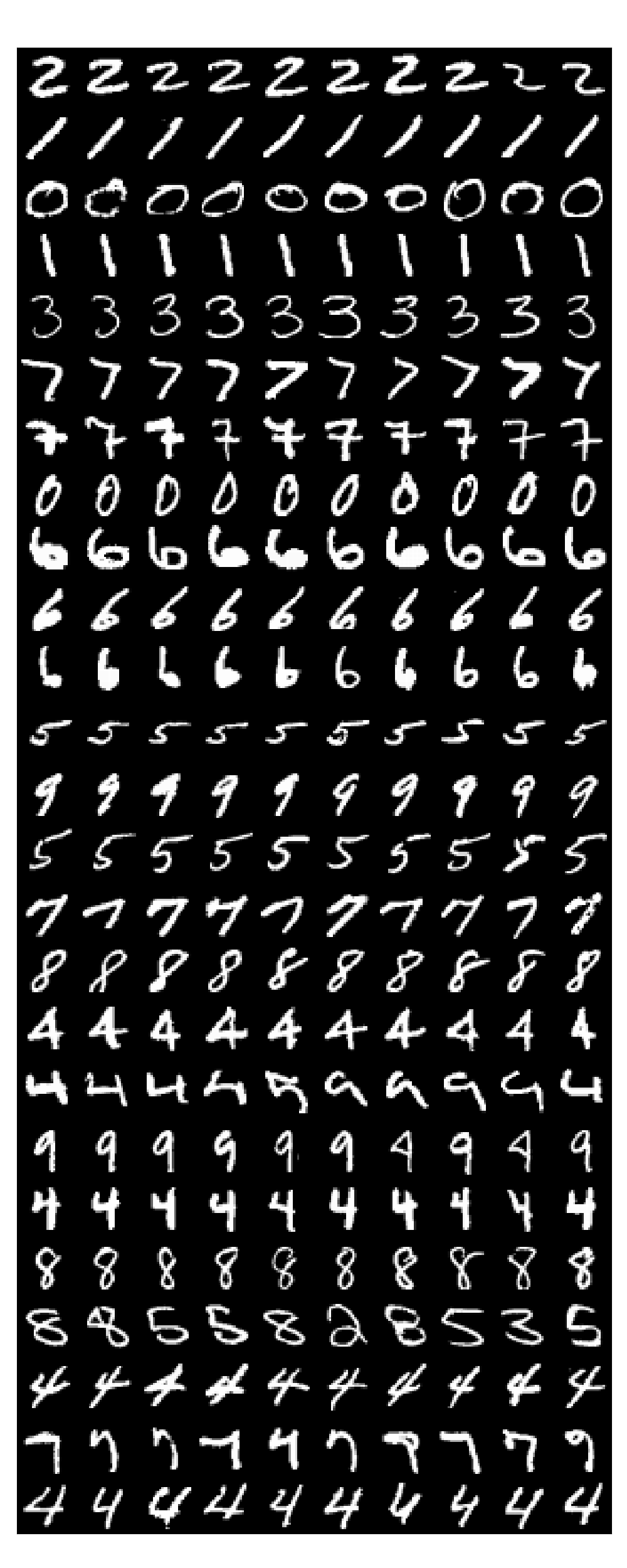}
        \caption{Arch2}\label{fig:gen2}
    \end{subfigure}
    \caption{Output of MFCVAE on MNIST data: Synthetically generated samples. Each row corresponds to a different learnt component. The columns are samples generated from the component. The rows are sorted by average confidence.}
    \label{fig: gen}
\end{figure}

\begin{figure}[htbp!]
    \centering
    \begin{subfigure}{.45\textwidth}
      \centering
      \includegraphics[width = \textwidth]{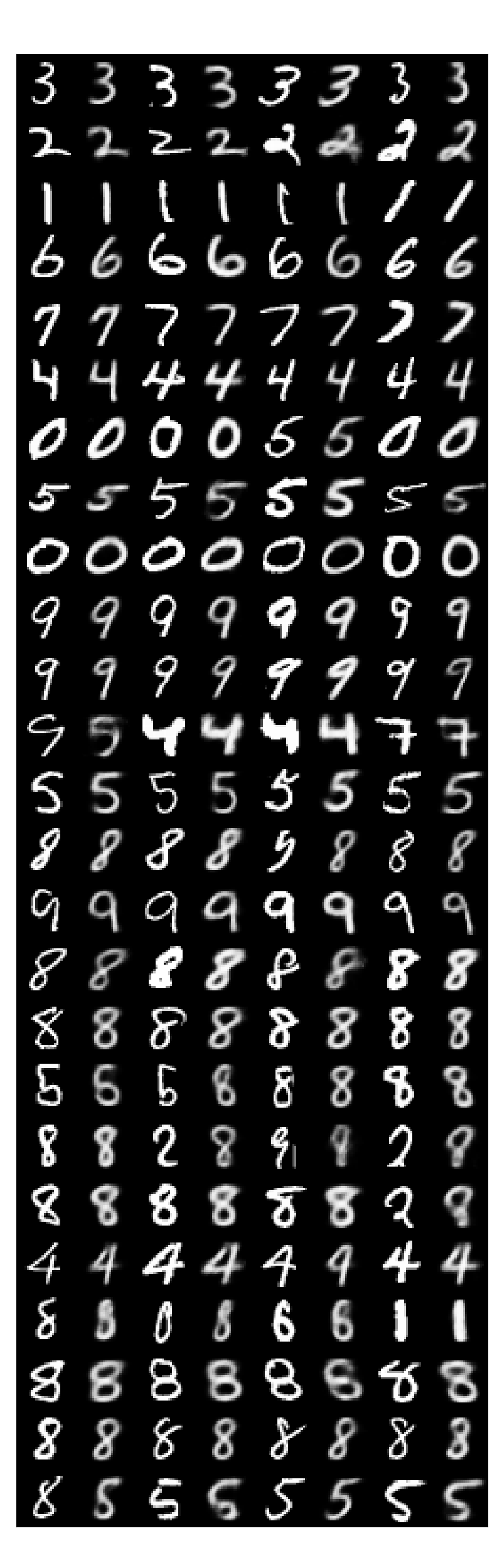}
        \caption{Arch1}\label{fig: recons1}
    \end{subfigure}
    \begin{subfigure}{.45\textwidth}
        \centering
        \includegraphics[width = \textwidth]{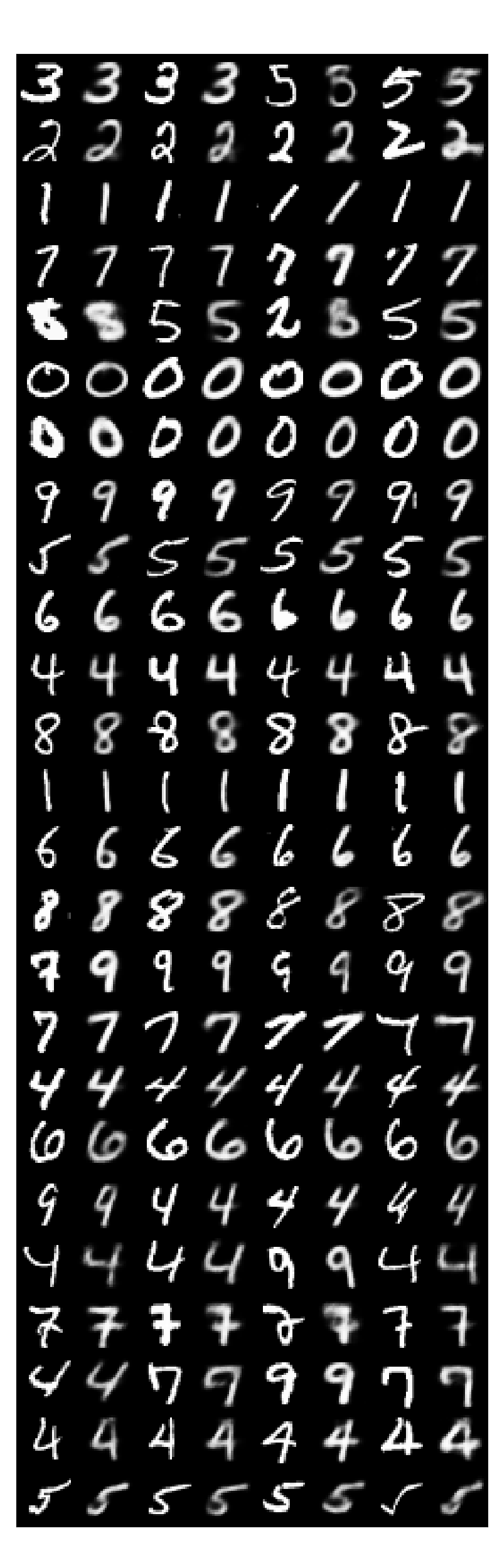}
        \caption{Arch2}\label{fig:recons2}
    \end{subfigure}
    \caption{Output of MFCVAE on MNIST data: Reconstruction accuracy. Each row corresponds to a different learnt component, the columns correspond to 4 different pairs of $x$ and $\widehat{x}$ in that order.}
    \label{fig: recons}
\end{figure}

\end{document}